%% file: main.tex
\documentclass[preprint]{article}

\usepackage[dvipsnames]{xcolor}
\usepackage{hyperref}
\usepackage{neurips_2026}
\usepackage{amssymb}
\usepackage{mathtools}          
\usepackage{mathrsfs}           
\usepackage{graphicx}           
\usepackage{subcaption}         
\usepackage[space]{grffile}     
\usepackage{url}                
\usepackage{lipsum}             
\usepackage{graphicx}
\usepackage{amsmath,amsfonts}
\usepackage{bm}
\usepackage{amsthm}
\usepackage{cleveref}
\usepackage{thmtools}
\usepackage{appendix}
\usepackage{microtype}
\usepackage{booktabs}
\usepackage{quiver}

\input{math_commands}

\hypersetup{
  linkcolor  = BrickRed,
  citecolor  = MidnightBlue,
  urlcolor   = Aquamarine,
  colorlinks = true,
}

\usepackage{tikz} 
\usepackage{tikzit}
\input{2cat.tikzstyles}

\allowdisplaybreaks

\title{World models of environment, agent and joint agent-environment systems}

\author{%
\begin{tabular}{c}
    Manuel Baltieri\textsuperscript{1,2} \quad
    Filippo Torresan\textsuperscript{1} \quad
    Yivan Zhang\textsuperscript{3}
    \\[0.25em]
    Alexander Boyd\textsuperscript{4} \quad
    Fernando E. Rosas\textsuperscript{2,5,6,7}
    \\[0.75em]
    \small
    \textsuperscript{1} Araya Inc.
    \\
    \small
    \textsuperscript{2} Department of Informatics, University of Sussex
    \\
    \small
    \textsuperscript{3} The University of Tokyo
    \\
    \small
    \textsuperscript{4} Beyond Institute for Theoretical Science (BITS)
    \\
    \small
    \textsuperscript{5} Sussex AI and Sussex Centre for Consciousness Science, University of Sussex
    \\
    \small
    \textsuperscript{6} Department of Brain Sciences, Imperial College London
    \\
    \small
    \textsuperscript{7} Centre for Eudaimonia and Human Flourishing, University of Oxford
    \\[0.75em]
    \small
    \texttt{manuel\_baltieri@araya.org} 
\end{tabular}
}

\begin{document}

% \makeCover  
\maketitle  

\begin{abstract}
    World models are a central component of model-based reinforcement learning.
    They are usually discussed in terms of what variables they predict, such as observations, rewards, states, latent or information states.
    We argue that there is a prior distinction: which channel they model.
    We consider three cases: the environment channel $\Observations_{:} \mid \Actions_{:}$, the agent channel $\Actions_{:} \mid \Observations_{:}$, and the realised joint process $(\Actions, \Observations)_{:}$, equivalently viewed as a channel with no inputs.
    Using computational mechanics, we define canonical predictive models for these three cases as $\epsilon$-transducers or $\epsilon$-machines.
    Canonical environment models recover standard predictive state representations, while the other two give analogous notions of canonical models for the agent and the joint system.
    We then build canonical support-restricted environment and agent models induced by closed-loop coupling, whose predictive equivalences range over continuations supported by the realised interaction.
    The key structural result is that canonical support-restricted environment states factor through the canonical joint causal states, and their transition structure is induced directly from the joint model; the agent-side construction is dual.
    Finally, we give a POMDP/controller example in which the unrestricted environment model has infinitely many states while the canonical support-restricted model induced by the coupling is finite.
    The framework clarifies what different world models are models of, and how coupling and support restriction can change their canonical predictive structure and complexity.
\end{abstract}

%============================================================
\section{Introduction}
%============================================================

World models are a central construct of modern (deep) reinforcement learning~\citep{haWorldModels2018, matsuoDeepLearningReinforcement2022, dingUnderstandingWorldPredicting2025, hafnerMasteringDiverseControl2025}.
Their formalisation usually relies on state or action-state abstractions, focusing mostly on fully-observable models~\citep{ravindranAlgebraicApproachAbstraction2004, abelTheoryAbstractionReinforcement2022}.
While often focused on modelling environments~\citep{niBridgingStateHistory2024, dingUnderstandingWorldPredicting2025}, world models of different kinds have increasingly emerged as relevant for agents to model themselves~\citep{haberLearningPlayIntrinsicallyMotivated2018, kanervistoWorldHumanAction2025}, other agents~\citep{zhouSocialWorldModels2025}, joint agent-environment ~\citep{meulemansEmbeddedUniversalPredictive2025} or even more complex systems~\citep{zhugeNeuralComputers2026}.

This proliferation suggests that world models should be distinguished not only by which variables they predict, see e.g.~\citep{niBridgingStateHistory2024} for world models of the environment, but also by which channel they model.
In the setting of agent-environment interaction, we consider three cases: the environment channel $\Observations_{:}\mid A_{:}$, the agent channel $\Actions_{:} \mid \Observations_{:}$, and the realised joint process $(\Actions, \Observations)_{:}$, equivalently viewed as the null-input channel $(\Actions, \Observations)_{:} \mid \mathbf 1_{:}$, i.e. a channel with trivial inputs ($\mathbf{1} = \{ \ast \}$).
These answer different predictive questions: how future observations depend on future actions, how future actions depend on future observations, and what future action-observation traces are generated by the coupled interaction itself.
They are not just different target variables for the same kind of model, as in e.g.~\citep{niBridgingStateHistory2024}.
They induce different causal equivalence relations, and therefore different notions of abstraction, model complexity, and use cases.

Distinguishing these channels and their corresponding models can be useful for different purposes.
For AI interpretability, understanding what RL agents effectively model, either by construction, or as a side effect of a particular architecture, loss function, or training method, is an open challenge~\citep{bereskaMechanisticInterpretabilityAI2024, dalrympleGuaranteedSafeAI2024}. 
This is particularly important for evaluating and supervising how RL agents form and use world models and beliefs of different kinds for decision making.
Making the channel explicit is especially important for interpretability.
Before asking whether the latent state learned by an AI model is a belief state, a predictive state, or an abstraction of some model, one must ask what it is a state \emph{of}: a model of the environment channel, the agent channel, or the realised joint process.
A latent state that is sufficient for predicting future observations under given actions need not be sufficient for predicting the agent's own future actions, and neither need coincide with a minimal state of the realised closed-loop interaction.

Various approaches, inspired by formal verification methods and logic~\citep{dalrympleGuaranteedSafeAI2024}, control theory~\citep{ullrichNewPerspectiveAI2025}, and neuroscience~\citep{mineaultNeuroAIAISafety2025}, among others, have been proposed to this end, taking advantage of different features and guarantees provided by each of these frameworks.
More recently, computational mechanics~\citep{crutchfieldCalculiEmergenceComputation1994, shaliziComputationalMechanicsPattern2001}, an approach with the goal of understanding the mechanics of computation and information processing, has provided a formal background and semantics to understand world models and abstractions for agents and machine learning models solving various tasks~\citep{shaiTransformersRepresentBelief2024, rosasAIVatFundamental2025, boydMonolithsModulesDecomposing2025, riechersNeuralNetworksLeverage2025, riechersNexttokenPretrainingImplies2025, piotrowskiConstrainedBeliefUpdates2025, shaiTransformersLearnFactored2026}.

In this work we extend the approach of~\citep{rosasAIVatFundamental2025} focused on the environment channel, and propose a systematic characterisation of canonical predictive models organised by the channel they model.
In~\cref{sec:canonicalWorldModels} we introduce canonical environment, agent, and joint models using $\epsilon$-machines and $\epsilon$-transducers, with more technical background in~\cref{sec:compMech}.
On the environment side, this recovers predictive state representations used to build models without relying on given latent variable representations, with predictive states as equivalence classes of histories are identified by the predictions they assign to future action-observation tests~\citep{littmanPredictiveRepresentationsState2001, singhPredictiveStateRepresentations2004}.
We apply the same criterion to the agent channel and to the realised joint process, giving canonical states for policy behaviour and coupled interaction as well as for environment dynamics.
In~\cref{sec:supportRestricted} we then show how closed-loop coupling induces canonical support-restricted environment and agent models, whose predictive equivalences range over continuations supported by the realised interaction.
This captures a precise interpretation of world models that represent only the possible interactions that an agent has with its environment.
These models are in general simpler because they don't include all the possible (counterfactual) ways an agent and environment could have interacted with under different circumstances.
The main structural result is that the non-sink canonical support-restricted environment states factor through the joint causal states, with transitions induced from the joint model, and the agent-side construction is dual.
This is not true in general for the case without support restriction, unless more assumptions are introduced.
In~\cref{sec:related} we review related work before a few final discussion points are brought out in~\cref{sec:conclusions}.

\paragraph{Contributions.}
Our contributions are three-fold.
First, we extend the viewpoint behind predictive state representations of environments to the other channels over the same interface of an agent-environment interaction, an agent channel and a joint process seen as a channel with trivial inputs, see~\cref{def:predictiveChannels}.
In this way, the (latent) states of a model are defined by different equivalence relations of action-observation histories based on their predictive content, rather than assumed to be given by a particular hidden-state presentation.
This gives an input-output analogue for policy behaviour and an autonomous predictive-state model for the realised joint process.
For reinforcement learning, this separates the information necessary for planning from that needed for modelling a policy or an on-policy rollout, by defining different different states depending on the task.
This formalises the informal idea that states sufficient for predicting future observations are in general not the same as those sufficient for predicting future actions, or future action-observation pairs.
For mechanistic interpretability, it gives instead reference targets for asking what kind of predictive state a learned representation implements.
This means, for instance, that we can extend existing work on beliefs about the environment of AI models such as transformers~\citep{shaiTransformersRepresentBelief2024, riechersNeuralNetworksLeverage2025, riechersNexttokenPretrainingImplies2025, piotrowskiConstrainedBeliefUpdates2025, shaiTransformersLearnFactored2026} to look for beliefs the AI model might have about itself or its interaction with the environment.
Second, we define canonical support-restricted environment and agent models induced by a fixed closed-loop coupling~\cref{def:canonicalSupportRestrictedEnvironmentModel,tab:environmentAgentAxes}.
This identifies when a model is being queried outside the support of the realised interaction, and distinguishes counterfactual world modelling from policy-conditioned prediction.
Third, we show that the support-restricted environment model is always determined by the joint model: the joint model determines its non-sink states and transition structure, with only the totalisation sink added~\cref{th:supportRestrictedEnvironmentFactors,th:transitionsEnvironmentInduced,th:jointReconstructsSupportRestrictedEnvironment}.
This shows why compact closed-loop prediction need not imply a compact unrestricted model.

\begin{figure}
    \centering
    \begin{subfigure}{0.3\textwidth}
        \includegraphics[width=\textwidth]{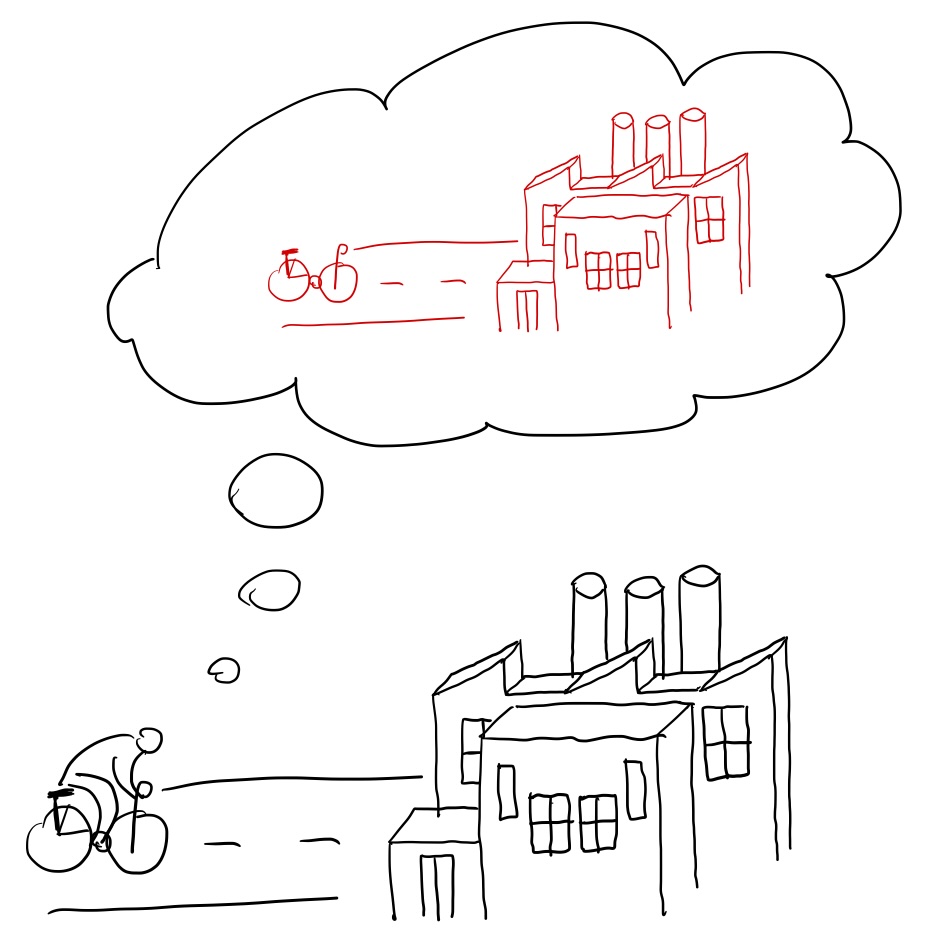}
        \caption{Environment model.}
        \label{fig:EnvModelPic}
    \end{subfigure}
    \hfill
    \begin{subfigure}{0.3\textwidth}
        \includegraphics[width=\textwidth]{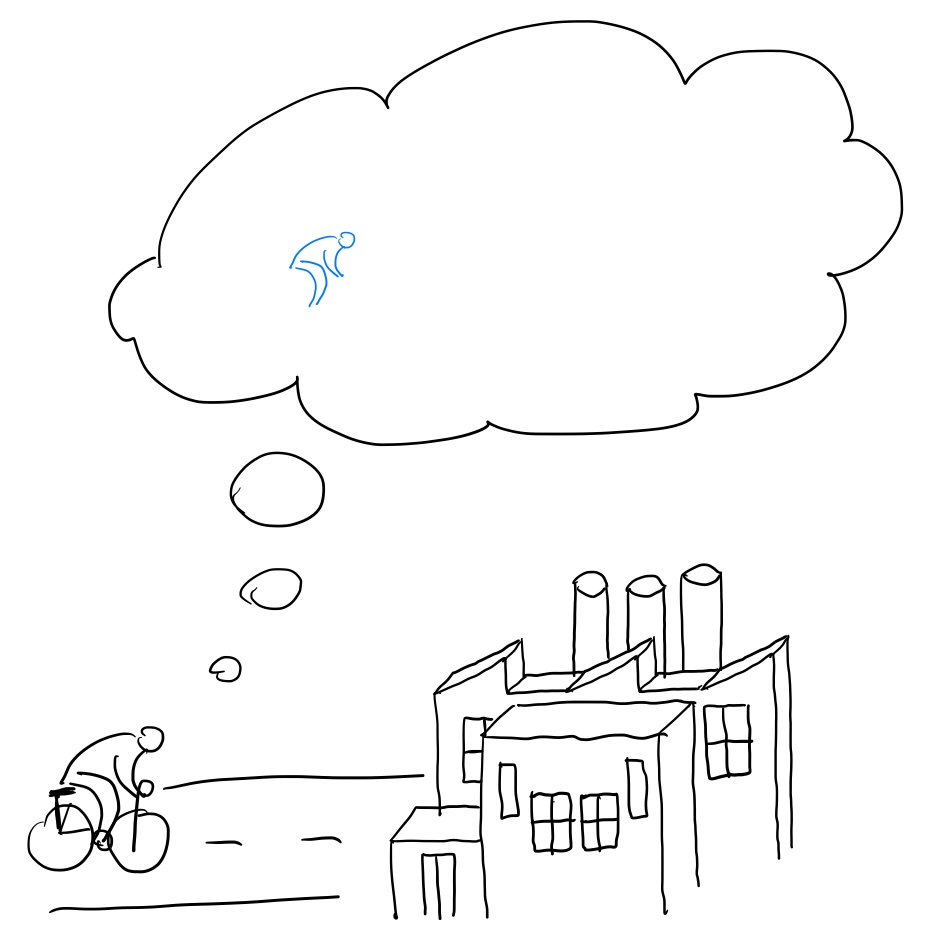}
        \caption{Agent model.}
        \label{fig:AgentModelPic}
    \end{subfigure}
    \hfill
    \begin{subfigure}{0.3\textwidth}
        \includegraphics[width=\textwidth]{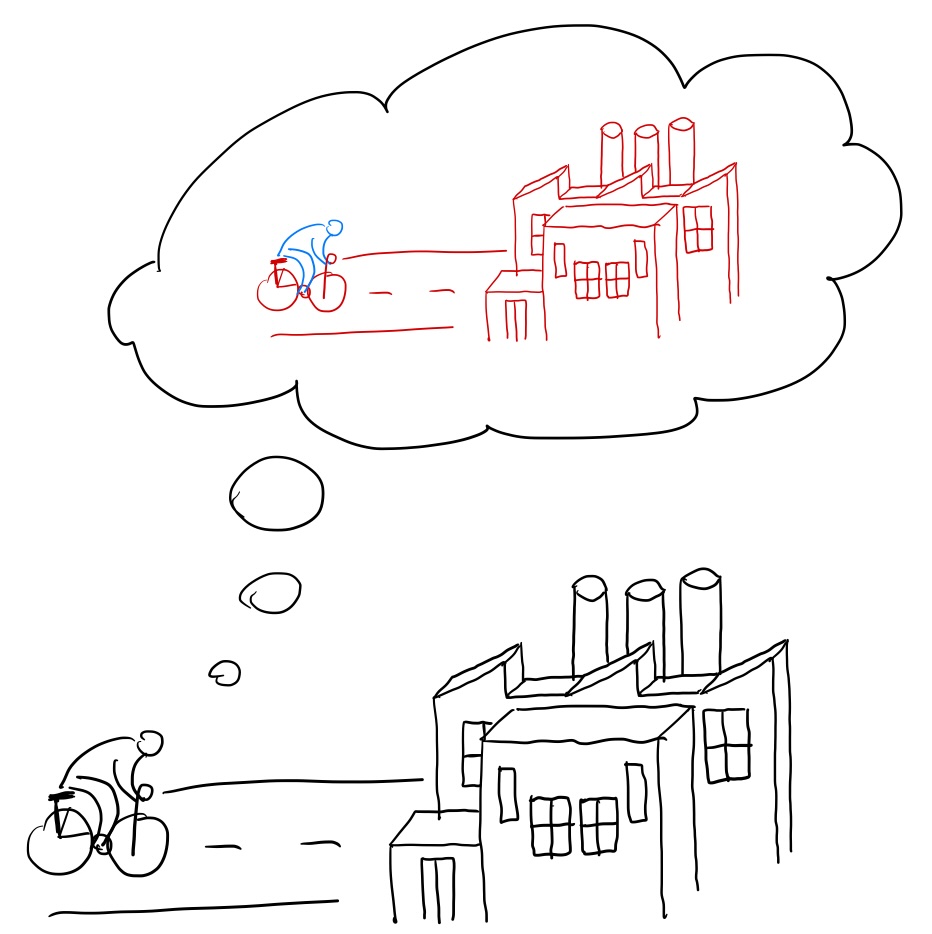}
        \caption{Joint model.}
        \label{fig:JointModelPic}
    \end{subfigure}
            
    \caption{Different kinds of world models, represented pictorially.
    The joint model represents the realised coupled action-observation process.}
    \label{fig:worldModelsPic}
\end{figure}

%============================================================
\section{Canonical models}
%============================================================
\label{sec:canonicalWorldModels}

This section introduces the mathematical setup used to characterize the interaction between an agent and its environment. 
Specifically, we formalise the predictive semantics of these interactions by defining canonical environment, agent, and joint models in terms of computational mechanics tools.

In the following, uppercase letters (e.g. $X,Y$) are used to denote random variables and lowercase (e.g. $x,y$) their realisations, calligraphic letters (e.g. $\mathcal{X}, \mathcal{Y}$) denote the sets over which they take values, and the symbol $\Delta$ (as in $\Delta(\mathcal{X}), \Delta(\mathcal{Y})$) is used to denote the collection of all distributions over those sets. 
We use the shorthand notation $p(x \mid y) = \Pr(X=x \mid Y=y)$ to express probabilities when there is no risk of ambiguity, and assume that equalities of the form $p(x \mid y,z) = p(x \mid y)$ hold for all realisations that can take place with non-zero probability. 
$\N = \{0,1,2, \ldots\}$ corresponds to zero-based numbering, and we use the following abbreviations: $\bm x_{t:t+L} = (x_{t} \dots, x_{t:t+L-1})$, $\bm x_{:t} = \bm x_{-\infty:t}$, $\bm x_{t:} = \bm x_{t:\infty}$, and $\bm x_: = \bm x_{-\infty:\infty}$. 
Note that the left index is always inclusive, the right always exclusive.
We will often treat an autonomous stochastic process $X_{:}$ as a channel with a trivial input alphabet $\mathbf{1} = \{ \ast \}$, identifying $X_{:}$ with the null-input channel $X_{:}\mid \mathbf{1}_{:}$ when convenient.

We initially consider how agents and environments interact over a given interface with observation symbols $\Observations_{t} \in \OBSERVATIONS$ and action symbols $\Actions_{t} \in \ACTIONS$. 
\begin{definition}[Agent-environment interface~\citep{myersCategoricalSystemsTheory2021, abelPlasticityMirrorEmpowerment2025}]
    The interface is a pair of finite sets, $(\ACTIONS, \OBSERVATIONS)$ such that $|\ACTIONS| \geq 2, |\OBSERVATIONS| \geq 2$.
\end{definition}

The predictive objects considered in this work can be described using channels: two channels as in~\citep{fidererWorkCapacityChannels2025}, together with the realised joint process viewed as a null-input channel.

\begin{definition}[Predictive channels over an agent-environment interface]
    \label{def:predictiveChannels}
    For an interaction over the interface $(\ACTIONS, \OBSERVATIONS)$, we consider three associated channels, the \emph{environment channel} $\Observations_{:} \mid \Actions_{:}$ which predicts observations from actions, the \emph{agent channel} $\Actions_{:}\mid\Observations_{:}$ which predicts actions from observations, and the realised joint process $(\Actions, \Observations)_{:}$, written as the null-input channel $(\Actions, \Observations)_{:}\mid\mathbf 1_{:}$.
\end{definition}

\begin{table}[ht!]
    \centering
    \begin{tabular}{clccl}
        \toprule
        Channel & Description & Input process & Output process & Goal \\
        \midrule
        $\Observations_{:} \mid \Actions_{:}$ & environment & $\Actions_{:}$ & $\Observations_{:}$ & predict observations from actions \\
        $\Actions_{:} \mid \Observations_{:}$ & agent & $\Observations_{:}$ & $\Actions_{:}$ & predict actions from observations \\
        $(\Actions, \Observations)_{:} \mid \mathbf{1}_{:}$ & joint & $\mathbf{1}_{:}$ & $(\Actions, \Observations)_{:}$ & predict realised coupled traces \\
        \bottomrule
    \end{tabular}
    \caption{Channels studied in this paper. 
    The channel, rather than only the predicted variables, determines what kind of world model is being defined.}
\end{table}

Next, we define the corresponding canonical models in terms of causal states for these three channels.
These should be seen as optimal, in the computational mechanics sense of unique, unifilar, minimal and predictive, world models for different channels, extending intuitions and classifications applied to environment models (\cref{fig:EnvModelPic}), see e.g.~\citep{niBridgingStateHistory2024}, to agent (\cref{fig:AgentModelPic}) and joint (\cref{fig:JointModelPic}) models.

\subsection{Canonical environment model}
\label{sec:canonicalEnvironment}

We now define the environment's channel causal states, $\STATES^\Environment$, see~\cref{sec:compMech}, with an equivalence relation $\sim_{\Environment}$ of histories of actions and observations.
For simplicity, we write $\HistoriesFiltering_{:t} = (\Actions, \Observations)_{:t}$ for action-observation histories, with realisations $\historiesfiltering_{:t}$ and history space $\HISTORIESFILTERING_{:t}$.
The equivalence relation is:
\begin{align}
    \label{eq:envCausalStates}
    \historiesfiltering_{:t} \sim_{\Environment} \historiesfiltering'_{:t} \iff \Pr(\Observations_{t:} \mid \Actions_{t:}, \HistoriesFiltering_{:t} = \historiesfiltering_{:t}) = \Pr(\Observations_{t:} \mid \Actions_{t:}, \HistoriesFiltering_{:t} = \historiesfiltering'_{:t})
\end{align}
and it induces a surjective map $\epsilon_\Environment : \HistoriesFiltering_{:t} \to \STATES^\Environment$ given by $\epsilon_\Environment(\historiesfiltering_{:t}) = \states^\Environment = \{ \historiesfiltering'_{:t} \mid \historiesfiltering_{:t} \sim_{\Environment} \historiesfiltering'_{:t} \}$.
These causal states are minimal sufficient statistics for prediction of future observations given future actions and past actions and observations, see~\cref{sec:epsTransducers}.
Notice how conditioning on $\Actions_{t:}$ in the environment channel is part of the predictive equivalence quantifier: it enforces equality of predictive behaviour for all possible future input continuations. 
This is particularly relevant when the environment is later coupled to an agent: the coupling may support only a subset of possible future input continuations, leading to a different partition of histories, see~\cref{sec:supportRestricted}.

The $\epsilon_\Environment$-map induces dynamics over causal states captured by stochastic matrices $\mathcal{T}^\Environment = \{ T^{(\observations \mid \actions)} \}_{\observations \in \OBSERVATIONS, \actions \in \ACTIONS}$ given by:
\begin{align}
    \label{eq:envTransitions}
    T^{(\observations \mid \actions)}_{\states^\Environment \to \bar{\states}^\Environment} = \Pr(\States^\Environment_{t+1} = \bar{\states}^\Environment, \Observations_{t} = \observations \mid \States^\Environment_{t} = \states^\Environment, \Actions_{t} = \actions).
\end{align}

\begin{definition}[Canonical environment model]
    A canonical environment model is the $\epsilon$-transducer $(\ACTIONS, \OBSERVATIONS, \STATES^\Environment, \mathcal{T}^\Environment)$ of the environment channel $\Observations_: \mid \Actions_:$.
\end{definition}

\begin{remark}[Relation to predictive state representations]
    The canonical environment model can be read as the all-tests, presentation-independent semantic for idealised predictive state representations.
    This relation has been noted in, among others, \cite{shaliziMethodsTechniquesComplex2006, zhangLearningCausalState2021, rosasAIVatFundamental2025}.
    In the predictive state representation literature, a test is a finite future action-observation experiment: future actions are supplied as inputs, and the model predicts the probability of the corresponding future observations~\citep{singhPredictiveStateRepresentations2004}.
    A predictive state representation represents a history by the predictions it assigns to such tests, here the causal state is the equivalence class of histories that agree on all future tests.
    Finite-dimensional predictive state representations typically choose a finite set of core tests or a particular presentation, whereas the $\epsilon$-transducer gives the optimal minimal unifilar presentation of the environment channel.
\end{remark}

\subsection{Canonical agent model}
\label{sec:canonicalAgent}

Consider an agent's causal states, $\STATES^\Agent$, defined by the equivalence relation $\sim_{\Agent}$:
\begin{align}
    \label{eq:agentCausalStates}
    \historiesfiltering_{:t} \sim_{\Agent} \historiesfiltering'_{:t} \iff \Pr(\Actions_{t+1:} \mid \Observations_{t:}, \HistoriesFiltering_{:t} = \historiesfiltering_{:t}) = \Pr(\Actions_{t+1:} \mid \Observations_{t:}, \HistoriesFiltering_{:t} = \historiesfiltering'_{:t})
\end{align}
inducing a surjective map $\epsilon_\Agent : \HistoriesFiltering_{:t} \to \STATES^\Agent$ given by $\epsilon_\Agent(\historiesfiltering_{:t}) = \states^\Agent = \{ \historiesfiltering'_{:t} \mid \historiesfiltering_{:t} \sim_{\Agent} \historiesfiltering'_{:t} \}$.
These causal states are minimal sufficient statistics for prediction of future actions given future observations and past actions and observations, see~\cref{sec:epsTransducers}.
The shift from $\Observations_{t}$ to $\Actions_{t+1}$ reflects the convention that the current observation is treated as the input for the agent's next action.
The $\epsilon_\Agent$-map induces dynamics over causal states captured by stochastic matrices $\mathcal{T}^\Agent = \{ T^{(\actions \mid \observations)} \}_{\actions \in \ACTIONS, \observations \in \OBSERVATIONS}$ given by:
\begin{align}
    \label{eq:agentTransitions}
    T^{(\actions \mid \observations)}_{\states^\Agent \to \bar{\states}^\Agent} = \Pr(\States^\Agent_{t+1} = \bar{\states}^\Agent, \Actions_{t+1} = \actions \mid \States^\Agent_{t} = \states^\Agent, \Observations_{t} = \observations).
\end{align}
\begin{definition}[Canonical agent model]
    We define a canonical agent model to be the $\epsilon$-transducer $(\OBSERVATIONS, \ACTIONS, \STATES^\Agent, \mathcal{T}^\Agent)$ of an agent channel $\Actions_: \mid \Observations_:$.
\end{definition}

The canonical agent model is obtained by applying the same predictive-state construction to the agent channel $\Actions_{:} \mid \Observations_{:}$ that predictive state representations apply to the environment channel $\Observations_{:} \mid \Actions_{:}$.

\begin{remark}[Action-predictive state representations]
    For the environment channel, a finite test fixes a future action word and asks for the probability of a future observation word.
    For the agent channel, the corresponding agent-side test fixes a future observation word and asks for the probability of a future action word.
    Concretely, for finite words $\observations_{t:t+L}$ and $\actions_{t+1:t+L+1}$,
    \begin{align}
        \Pr(
            \Actions_{t+1:t+L+1}=\actions_{t+1:t+L+1}
            \mid
            \Observations_{t:t+L}=\observations_{t:t+L},
            \HistoriesFiltering_{:t}=\historiesfiltering_{:t}
        ).
    \end{align}
    Thus two histories are agent-causally equivalent when they give the same predictions for all such future observation-action tests.
    In this sense, the canonical agent model is a predictive-state model of policy or controller behaviour: it represents histories by how the agent would act under all possible future observation continuations.
\end{remark}

As in the case of the environment channel, conditioning on $\Observations_{t:}$ is part of the predictive equivalence quantifier: it compares predictive behaviour across all possible future input (observation) continuations. 
This does not mean that the agent observes the future.
In~\cref{sec:supportRestricted}, we instead consider support-restricted variants, where only observation continuations supported by the coupled environment are considered.
If one further weights or selects observation continuations by desirability, related constructions connect to control-as-inference~\citep{kappenOptimalControlGraphical2012}, planning-as-inference~\citep{levineReinforcementLearningControl2018}, and active-inference perspectives~\citep{baltieriKalmanBucyFiltersLinear2020, millidgeRelationshipActiveInference2020}.

\subsection{Canonical joint model}
\label{sec:canonicalJointModel}
Finally, we consider joint agent-environment models.
This is inspired by the idea of embedded agency~\citep{demskiEmbeddedAgency2020}, where agents model themselves as part of a world that includes them and their environment.
Following the convention introduced above, the joint process can also be viewed as a null-input channel, $(\Actions, \Observations)_{:} \mid \mathbf{1}_{:}$.
Unlike the environment and agent channels, however, it has no non-trivial exogenous input: it describes the realised coupled action-observation autonomous process.
We start defining an equivalence relation $\sim_{\Joint}$ of the joint pasts given by:
\begin{align}
    \label{eq:jointCausalStates}
    \historiesfiltering_{:t} \sim_{\Joint} \historiesfiltering'_{:t} \iff \Pr((\Actions, \Observations)_{t:} \mid \HistoriesFiltering_{:t} = \historiesfiltering_{:t}) = \Pr((\Actions, \Observations)_{t:} \mid \HistoriesFiltering_{:t} = \historiesfiltering'_{:t})
\end{align}
which gives a surjective map $\epsilon_\Joint : (\ACTIONS \times \OBSERVATIONS)_{:t} \to \STATES^\Joint$ defined by $\epsilon_\Joint(\historiesfiltering_{:t}) = \states^\Joint = \{ \historiesfiltering'_{:t} \mid \historiesfiltering_{:t} \sim_{\Joint} \historiesfiltering'_{:t} \}$.
The resulting states $\STATES^\Joint$ are the causal states of the joint $\epsilon$-machine~\citep{barnettComputationalMechanicsInput2015}.
Transitions between these causal states are given by stochastic matrices $\mathcal{T}^\Joint = \{ T^{(\actions, \observations)} \}_{(\actions, \observations) \in \ACTIONS \times \OBSERVATIONS}$:
\begin{align}
    \label{eq:jointTransitions}
    T^{(\actions, \observations)}_{\states^\Joint \to \bar{\states}^\Joint} = \Pr(\States^\Joint_{t+1} = \bar{\states}^\Joint, (\Actions, \Observations)_{t} = (\actions, \observations) \mid \States^\Joint_{t} = \states^\Joint).
\end{align}
\begin{definition}[Canonical joint model]
    \label{def:canonicalJointModel}
    A canonical joint model of the joint process $(\Actions, \Observations)_{:}$ is the unique minimal unifilar machine presentation given by the tuple $(\ACTIONS \times \OBSERVATIONS, \STATES^\Joint, \mathcal{T}^\Joint)$.
\end{definition}

\begin{remark}[Joint predictive states and policy-induced processes]
    Under the null-input-channel convention, the joint model is the autonomous counterpart of the environment and agent channel models.
    Its causal states represent histories by their predictions of future action-observation traces, rather than by predictions of future observations under action inputs or future actions under observation inputs.
    In a fully observable MDP with a fixed policy, this reduces to the familiar policy-induced Markov chain, or Markov reward process if rewards are included~\citep{suttonReinforcementLearningIntroduction2018, putermanMarkovDecisionProcesses2014}.
    More generally, coupling a partially observable environment to a policy or finite-state controller induces an autonomous stochastic process over action-observation traces~\citep{meuleauLearningFinitestateControllers1999, koulLearningFiniteState2018}.
    The joint model is the canonical minimal predictive presentation of this realised process.
\end{remark}

\subsection{Relations among the three canonical models}
\label{sec:canonicalRelations}

Since the environment, agent, and joint models are all built from the same action-observation histories, one might expect their causal states to coincide, or at least to be related by a simple construction.
This is not true in general.
The reason is that the three models are minimal for different prediction problems: the environment model predicts future observations under future action inputs, the agent model predicts future actions under future observation inputs, while the joint model predicts future action-observation traces of the realised coupled process.
These are different equivalence relations on histories, cf.~\cref{eq:envCausalStates,eq:agentCausalStates,eq:jointCausalStates}, and therefore need not induce the same partition.
We introduce a running example to illustrate this point.

The example is deliberately minimal, but it isolates a common pattern in reinforcement learning.
In partially observable environments, arbitrary (counterfactual) action continuations can require tracking a rich belief state, while a particular policy, controller, action mask, option structure, or dataset support may realise only a much smaller language of continuations.
Thus a model of the environment under arbitrary action inputs can have a different predictive structure from a model of the action-observation traces generated by a fixed coupling.
This means that a compact representation learned from rollouts need not be a representation of the full counterfactual environment dynamics; it may instead encode the realised policy-environment process.
The binary construction below is a small illustration of this phenomenon.
\begin{example}
    \label{ex:Example}
    Consider a binary environment with latent state $E_{t}\in \{0,1\}$, binary actions, and binary observations.
    Action $0$ resets the next latent state to a fair coin, while action $1$ holds the latent state fixed.
    The observation is a noisy readout of the new latent state, with noise parameter $\eta\in(0,1/2)$.
    For the unrestricted environment channel $\Observations_{:}\mid\Actions_{:}$, arbitrary future action continuations must be considered.
    In particular, arbitrarily long runs of action $1$ generate infinitely many distinct posterior beliefs over $E_{t}$.
    These posterior beliefs give different predictions for future observations, and hence correspond to infinitely many distinct environment causal states.
    Now couple this environment to a deterministic finite-state controller with memory states $\{\alpha, \beta, \gamma\}$.
    The controller emits action $0$ in state $\alpha$, and action $1$ in states $\beta$ and $\gamma$.
    Its memory update is such that, after a reset, it performs one hold action and sometimes one further hold action before returning to reset.
    As an agent channel $\Actions_{:}\mid\Observations_{:}$, this controller has only finitely many causal states: its finitely many memory modes are predictively sufficient for future action prediction.
    The realised joint process is also finite-state.
    Once the controller is coupled to the environment, arbitrarily long hold continuations are no longer realised.
    Only finitely many posterior configurations appear in the coupled action-observation process, and the joint $\epsilon$-machine has finitely many causal states.
    Thus, in this example, the unrestricted environment model is infinite, the agent model is finite, and the joint model is finite.
\end{example}

In this example, we can see that the joint model cannot easily obtained simply by identifying the causal states of the environment and agent models, because these are sufficient statistics of histories for different future predictions, nor by taking an obvious product, sum, or monotone combination of their state spaces.
A product bound can, for instance, be recovered after adding extra structure, see~\cref{sec:extraStructure}.
In general however, the separately defined environment and agent causal states need not retain the correlation, timing, or support information required for the pair $(\States^\Environment_{t}, \States^\Agent_{t})$ to be sufficient for predicting the realised joint process.
Moreover, in examples such as the one above, the bound is not informative because the unrestricted environment model is already infinite.

The next section shows that a more robust relation appears after support restriction.
While the unrestricted models are not related in any obvious way, the canonical support-restricted environment and agent models induced by the realised coupling do factor through the joint model.

\section{Canonical support-restricted models from closed-loop coupling}
\label{sec:supportRestricted}
The canonical environment and agent models of~\cref{sec:canonicalWorldModels} are \emph{unrestricted} channel models: their causal equivalences compare predictive behaviour under arbitrary future input continuations.
In a realised closed-loop agent-environment interaction, however, not all such continuations are supported~\citep{heldMovementproducedStimulationDevelopment1963, buckleyTheoryHowActive2018}.
An agent may be limited by its policy class, controller, embodiment, action mask, or choices, see~\cref{fig:closedLoop}.
Dually, an environment restricts which future observation continuations are available to the agent.
We call the resulting models \emph{canonical support-restricted models}.
We first give the environment-side construction.
The agent-side construction is obtained dually by exchanging actions and observations.

\begin{figure}[ht!]
    \centering
    \begin{subfigure}{0.48\textwidth}
        \includegraphics[width=\textwidth]{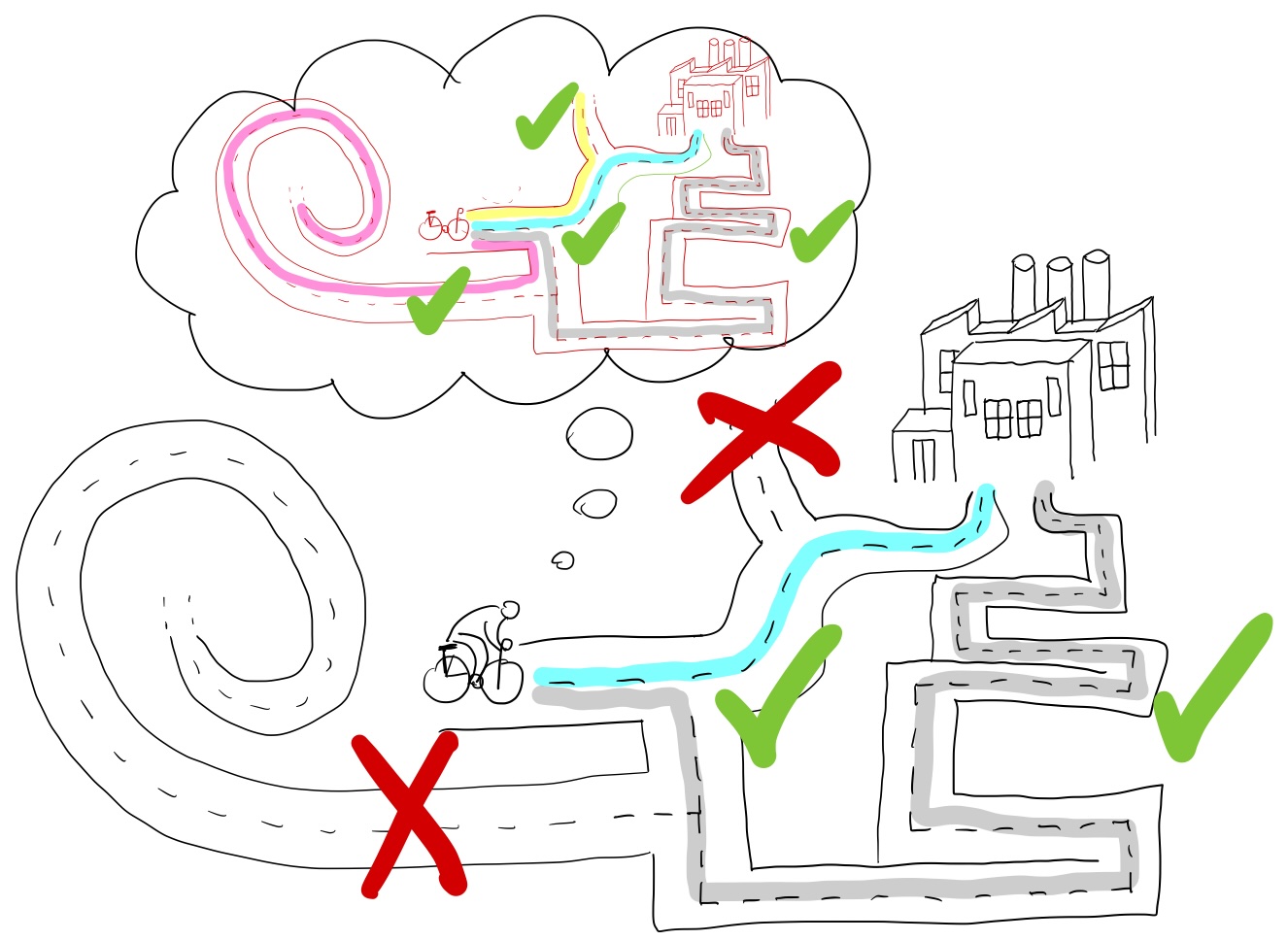}
        \caption{Canonical environment model.}
        \label{fig:unrestricted}
    \end{subfigure}
    % \hfill
    \begin{subfigure}{0.48\textwidth}
        \includegraphics[width=\textwidth]{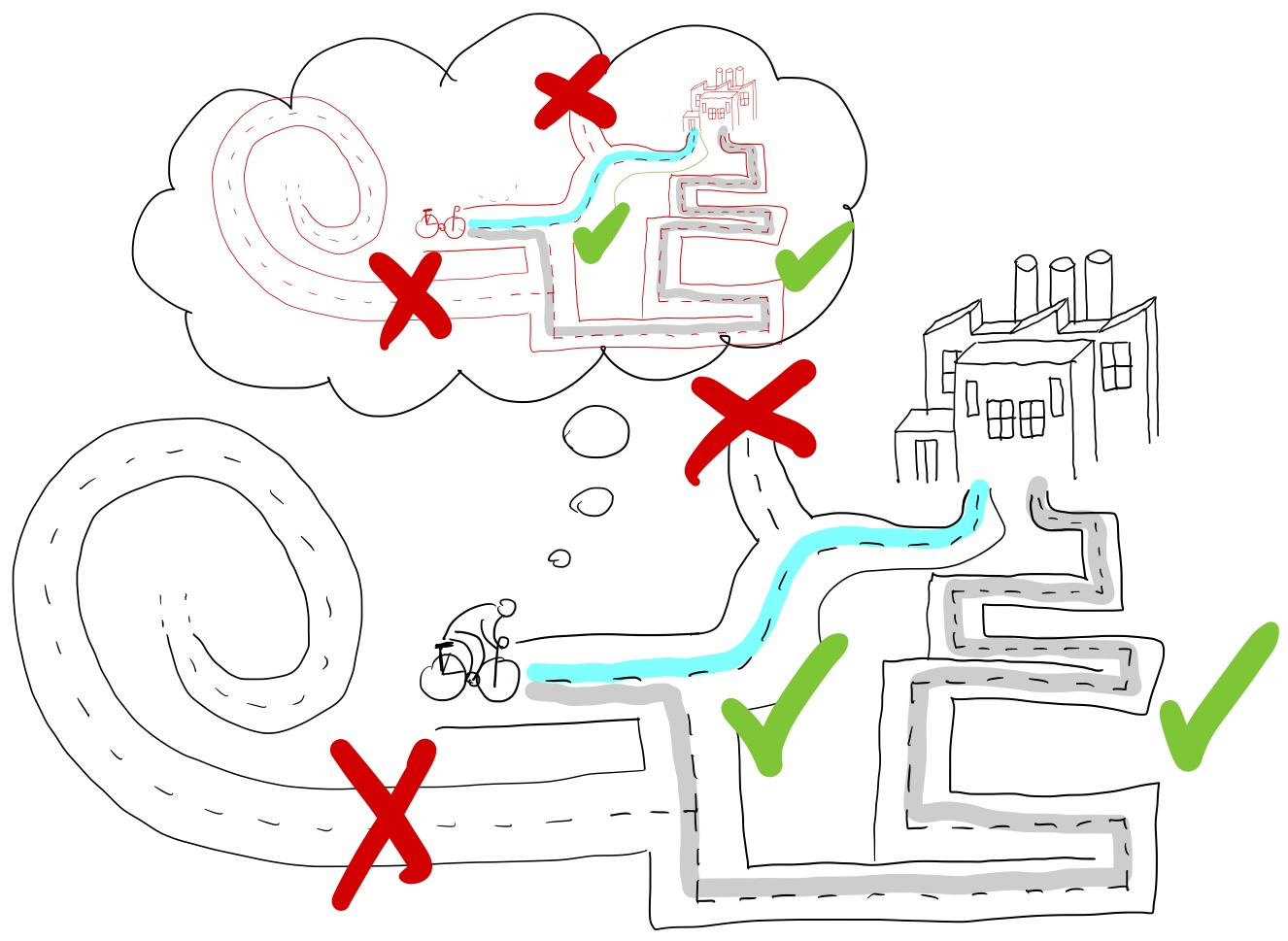}
        \caption{Canonical support-restricted environment model.}
        \label{fig:restricted}
    \end{subfigure}
    \caption{Unrestricted and support-restricted prediction.
    The unrestricted model considers arbitrary future continuations through the interface, including counterfactual paths not generated by the coupled agent, i.e. for instance paths on a map that an agent doesn't consider or never takes.
    The canonical support-restricted model keeps only continuations supported by the realised coupling: unsupported queries are marked as failures and sent to the totalisation sink $\bot$.}
    \label{fig:closedLoop}
\end{figure}

The natural support-restricted environment object is initially partial.
For an ordinary history, the realised coupling determines which action queries are supported.
On supported action queries, the environment-side prediction is the ordinary conditional observation law induced by the joint process.
On unsupported action queries, no ordinary environment prediction is defined by the realised coupling.
We totalise~\citep{silvaGeneralizingDeterminizationAutomata2013, jacobsIntroductionCoalgebraMathematics2017} this partial channel by adjoining a symbol $\bot$, which is emitted by the totalised channel when an action query is unsupported.
Intuitively, an action can be unsupported because the coupled agent cannot execute it, its controller or policy assigns it zero probability, or it is excluded by an action mask or other constraint.

Let $\widetilde{\OBSERVATIONS} \coloneqq \OBSERVATIONS \cup \{\bot\}$ be the augmented observation alphabet, and let the set of extended histories be $\widetilde{\HISTORIESFILTERING}_{:t} \coloneqq (\ACTIONS \times \widetilde{\OBSERVATIONS})_{:t}$.
Call $\widetilde \historiesfiltering_{:t}\in \widetilde{\HISTORIESFILTERING}_{:t}$ \emph{ordinary} if it contains no
occurrence of $\bot$.
In that case we identify it with the unique ordinary history $\historiesfiltering_{:t} \in \HISTORIESFILTERING_{:t}$ having the same symbols.
To make the domain of this partial channel explicit, we first define which finite future action continuations are supported by the realised joint process at a given ordinary history.
This finite-horizon definition describes the support language of the coupling.
The one-step case is then used to define the transition kernel below, because the support-restricted channel is specified one queried action at a time.

\begin{definition}[Realisable action continuations]
    \label{def:realisableActions}
    For an ordinary history $\historiesfiltering_{:t} \in \HISTORIESFILTERING_{:t}$ and $k\in \mathbb N^+$, the set of $k$-step realisable action continuations is
    \begin{align}
        \ACTIONS_k(\historiesfiltering_{:t})
        \coloneqq
        \left\{
            \actions_{t:t+k}\in \ACTIONS^k
            \;\middle|\;
            \exists \observations_{t:t+k}\in \OBSERVATIONS^k:
            \Pr\bigl(
                (\Actions, \Observations)_{t:t+k}
                =
                (\actions, \observations)_{t:t+k}
                \mid
                \HistoriesFiltering_{:t}=\historiesfiltering_{:t}
            \bigr) > 0
        \right\}.
    \end{align}
    An infinite future action continuation is realisable, in the cylinder-support sense, when all of its finite prefixes are realisable:
    $
        \actions_{t:} \in \ACTIONS_\infty(\historiesfiltering_{:t})
        \iff
        \actions_{t:t+k} \in \ACTIONS_k(\historiesfiltering_{:t})
        \quad\text{for all } k \in \N^+.
    $
    The one-step realisable action set, used in the transition kernel below, is the case $k=1$:
    \begin{align}
        \ACTIONS_1(\historiesfiltering_{:t})
        =
        \left\{
            \actions \in \ACTIONS
            \;\middle|\;
            \exists \observations \in \OBSERVATIONS:
            \Pr\bigl(
                (\Actions, \Observations)_{t} = (\actions, \observations)
                \mid
                \HistoriesFiltering_{:t} = \historiesfiltering_{:t}
            \bigr) > 0
        \right\}.
    \end{align}
    Equivalently, with the marginal
    $
        q(\actions \mid \historiesfiltering_{:t})
        \coloneqq
        \sum_{\observations \in \OBSERVATIONS}
        \Pr\bigl(
            (\Actions, \Observations)_{t} = (\actions, \observations)
            \mid
            \HistoriesFiltering_{:t} = \historiesfiltering_{:t}
        \bigr),
    $
    we have
    $
        \actions \in \ACTIONS_1(\historiesfiltering_{:t})
        \iff
        q(\actions \mid \historiesfiltering_{:t}) > 0.
    $
\end{definition}
The finite-horizon sets $\ACTIONS_k(\historiesfiltering_{:t})$ describe the future action support induced by the realised coupling.
The one-step set $\ACTIONS_1(\historiesfiltering_{:t})$ is enough for the recursive transducer definition: after a supported action and an ordinary observation, support is recomputed at the next step.
If a queried future action sequence first leaves this support, the totalised channel emits $\bot$ at that step and then remains in the sink.
For $k = 1$, this recovers the usual on-policy action support when the realised joint process is generated by a fixed policy coupled to an environment.
\begin{example}[On-policy one-step actions]
    \label{ex:onPolicyRealisableActions}
    Suppose the realised joint process is generated by coupling an environment with a policy
    $
        \pi(\actions\mid\historiesfiltering_{:t})
        =
        \Pr(\Actions_{t}=\actions\mid\HistoriesFiltering_{:t}=\historiesfiltering_{:t}).
    $
    If the one-step joint law factors as
    $
        \Pr\bigl((\Actions, \Observations)_{t}=(\actions, \observations)
        \mid
        \HistoriesFiltering_{:t}=\historiesfiltering_{:t}\bigr)
        =
        \pi(\actions\mid\historiesfiltering_{:t})
        \Pr\bigl(\Observations_{t}=\observations
        \mid
        \HistoriesFiltering_{:t}=\historiesfiltering_{:t},
        \Actions_{t}=\actions\bigr),
    $
    then
    \begin{align}
        \ACTIONS_1(\historiesfiltering_{:t})
        =
        \left\{
            \actions \in \ACTIONS
            \;\middle|\;
            \pi(\actions\mid\historiesfiltering_{:t}) > 0
            \text{ and }
            \exists\observations \in \OBSERVATIONS:
            \Pr(\Observations_{t} = \observations
            \mid
            \HistoriesFiltering_{:t} = \historiesfiltering_{:t},
            \Actions_{t} = \actions) > 0
        \right\}.
    \end{align}
    In the usual case where the environment produces some observation with nonzero probability after every action in the support of $\pi$, this reduces to
    \begin{align}
        \ACTIONS_1(\historiesfiltering_{:t})
        =
        \supp \pi(\cdot\mid\historiesfiltering_{:t}).
    \end{align}
    Thus, for a deterministic policy, the one-step realisable action set is the singleton containing the on-policy action:
    $
        \ACTIONS_1(\historiesfiltering_{:t})
        =
        \{\pi(\historiesfiltering_{:t})\}.
    $
\end{example}

\begin{definition}[Support-restricted environment channel]
    \label{def:supportRestrictedEnv}
    The support-restricted environment channel $\widetilde \Observations_: \mid \Actions_:$ is defined on extended histories
    $\widetilde \historiesfiltering_{:t} \in \widetilde{\HISTORIESFILTERING}_{:t}$ as follows:
    \begin{enumerate}
        \item if $\widetilde \historiesfiltering_{:t}$ contains an occurrence of $\bot$, then for every
        $\actions \in \ACTIONS, \Pr \bigl(\widetilde \Observations_{t} = \bot \mid \Actions_{t} = \actions, \widetilde \HistoriesFiltering_{:t} = \widetilde \historiesfiltering_{:t} \bigr) = 1$.
        \item if $\widetilde \historiesfiltering_{:t}$ is ordinary, identified with $\historiesfiltering_{:t} \in \HISTORIESFILTERING_{:t}$, then for every $\actions \in \ACTIONS$ and $\widetilde{\observations} \in \widetilde{\OBSERVATIONS}$,
        \begin{align*}
            & \Pr\bigl(\widetilde{\Observations}_{t} = \widetilde{\observations} \mid \Actions_{t} = \actions, \widetilde{\HistoriesFiltering}_{:t} = \historiesfiltering_{:t}\bigr) \nonumber \\
            \coloneqq
            &
            \begin{cases}
                \dfrac{\Pr\bigl((\Actions, \Observations)_{t}=(\actions, \observations) \mid \HistoriesFiltering_{:t} = \historiesfiltering_{:t} \bigr)}{\sum_{\bar{\observations} \in \OBSERVATIONS}\Pr\bigl((\Actions, \Observations)_{t} =(\actions, \bar{\observations}) \mid \HistoriesFiltering_{:t} = \historiesfiltering_{:t}\bigr)},
                & \text{if } \widetilde{\observations} = \observations \in \OBSERVATIONS\text{ and } \actions \in \ACTIONS_1(\historiesfiltering_{:t}), \\[2.0ex]
                1,
                & \text{if } \widetilde{\observations} = \bot\text{ and } \actions \notin \ACTIONS_1(\historiesfiltering_{:t}), \\
                0,
                & \text{otherwise.}
            \end{cases}
        \end{align*}
    \end{enumerate}
\end{definition}

\begin{proposition}[Totalisation of the partial channel]
    The channel of~\cref{def:supportRestrictedEnv} is the totalisation of the partial support-restricted environment channel induced by the realised joint process.
    On realisable actions it agrees with the ordinary conditional observation law induced by the joint process.
    On unrealisable actions it emits $\bot$, and after $\bot$ has occurred all future outputs are $\bot$.
\end{proposition}

\begin{proof}
    This follows directly from the two clauses of~\cref{def:supportRestrictedEnv}.
\end{proof}

We now define the canonical predictive states of this channel in the standard way.
For extended histories $\widetilde \historiesfiltering_{:t}, \widetilde \historiesfiltering'_{:t}\in \widetilde{\HISTORIESFILTERING}_{:t}$, write
\begin{align}
    \widetilde \historiesfiltering_{:t}\sim_{\epsilon_{\mathrm{sr}-\Environment}}\widetilde \historiesfiltering'_{:t}
    \iff
    \Pr\bigl(\widetilde \Observations_{t:}\mid \Actions_{t:}, \widetilde \HistoriesFiltering_{:t}=\widetilde \historiesfiltering_{:t}\bigr)
    =
    \Pr\bigl(\widetilde \Observations_{t:}\mid \Actions_{t:}, \widetilde \HistoriesFiltering_{:t}=\widetilde \historiesfiltering'_{:t}\bigr).
\end{align}
Let $\epsilon_{\mathrm{sr}-\Environment}(\widetilde \historiesfiltering_{:t})
\coloneqq
[\widetilde \historiesfiltering_{:t}]_{\sim_{\epsilon_{\mathrm{sr}-\Environment}}}$ and $\STATES^{\mathrm{sr}-\Environment}
\coloneqq
\widetilde{\HISTORIESFILTERING}_{:t}/\sim_{\epsilon_{\mathrm{sr}-\Environment}}$.
The induced transition matrices are
\begin{align}
    T^{\mathrm{sr}-\Environment, (\widetilde{\observations} \mid \actions)}_{\states \to \states'}
    \coloneqq
    \Pr\bigl(\States^{\mathrm{sr}-\Environment}_{t+1} = \states', \widetilde \Observations_{t} = \widetilde{\observations}
\mid \States^{\mathrm{sr}-\Environment}_{t} = \states, \Actions_{t} = \actions\bigr).
\end{align}

\begin{definition}[Canonical support-restricted environment model]
    \label{def:canonicalSupportRestrictedEnvironmentModel}
    The canonical support-restricted environment model is the $\epsilon$-transducer $(\ACTIONS, \widetilde{\OBSERVATIONS}, \STATES^{\mathrm{sr}-\Environment}, \mathcal T^{\mathrm{sr}-\Environment})$ of the support-restricted environment channel $\widetilde \Observations_: \mid \Actions_:$.
\end{definition}
This $\epsilon$-transducer always contains a \emph{sink state}, which is the equivalence class of all extended histories containing at least one occurrence of $\bot$, generated by an unsupported action query.
\begin{proposition}[The totalisation sink]
    \label{prop:sinkState}
    All extended histories containing at least one occurrence of $\bot$ belong to the same causal state, denoted by $\states^\Environment_\bot$. 
    It is absorbing, and for every $\actions \in \ACTIONS$, $T^{\mathrm{sr}-\Environment, (\bot \mid \actions)}_{\states^\Environment_\bot\to \states^\Environment_\bot} = 1$.
\end{proposition}
\begin{proof}
    See~\cref{sec:proofs}.
\end{proof}

The canonical support-restricted agent model is defined dually, by restricting future observation continuations to those supported by the realised coupling and totalising unsupported observation queries with an agent-side sink.
Interestingly, the restrictions of future actions for the environment and of future observations for the agent model, which use information from the joint model, lead to some interesting relations among these models, as proven next.

\subsection{Relations among canonical support-restricted models and the canonical joint model}
Here we look at the relations among these canonical support-restricted models and the canonical joint model, proving a result that cannot be obtained for their unrestricted counterparts without extra structure (e.g. a factorisation assumptions) as seen in~\cref{sec:extraStructure}.
First we show that the non-sink states of the canonical support-restricted environment are induced by the causal states of the joint process.
\begin{theorem}[Canonical support-restricted environment states factor through the joint $\epsilon$-machine causal states]
    \label{th:supportRestrictedEnvironmentFactors}
    Let $\epsilon_\Joint: \HISTORIESFILTERING_{:t} \to \STATES^\Joint$
    be the causal-state map of the canonical joint model. 
    If ordinary histories $\historiesfiltering_{:t}, \historiesfiltering'_{:t} \in \HISTORIESFILTERING_{:t}$ satisfy $\epsilon_\Joint(\historiesfiltering_{:t}) = \epsilon_\Joint(\historiesfiltering'_{:t})$, this implies that $\epsilon_{\mathrm{sr}-\Environment}(\historiesfiltering_{:t}) = \epsilon_{\mathrm{sr}-\Environment}(\historiesfiltering'_{:t})$.
    In other words, there exists a unique surjective map $\psi: \STATES^\Joint \to \STATES^{\mathrm{sr}-\Environment} \setminus \{\states^\Environment_\bot\}$ such that for every ordinary history $\historiesfiltering_{:t} \in \HISTORIESFILTERING_{:t}$, $\epsilon_{\mathrm{sr}-\Environment}(\historiesfiltering_{:t}) = \psi \bigl(\epsilon_\Joint(\historiesfiltering_{:t})\bigr)$, i.e. 
    % $\epsilon_{\mathrm{sr}-\Environment} = \psi \circ \epsilon_\Joint$.
    \begin{equation}
        % https://q.uiver.app/#q=WzAsMyxbMCwwLCJcXGJ1bGxldCJdLFsxLDAsIlxcYnVsbGV0Il0sWzEsMSwiXFxidWxsZXQiXSxbMCwxLCJhIl0sWzAsMiwiYyIsMl0sWzEsMiwiYiJdXQ==
        \begin{tikzcd}
        	{\HISTORIESFILTERING_{:t}} & {\STATES^\Joint} \\
        	& {\STATES^{\mathrm{sr}-\Environment} \setminus \{\states^\Environment_\bot\}}
        	\arrow["\epsilon_\Joint", two heads, from=1-1, to=1-2]
        	\arrow["\epsilon_{\mathrm{sr}-\Environment}"', two heads, from=1-1, to=2-2]
        	\arrow["\psi", two heads, from=1-2, to=2-2]
        \end{tikzcd}
    \end{equation}
\end{theorem}

\begin{proof}
    See~\cref{sec:proofs}.
\end{proof}

This gives
$
    |\STATES^{\mathrm{sr}-\Environment}\setminus\{\states^\Environment_\bot\}|
    \leq
    |\STATES^\Joint|,
$
i.e.
$
    |\STATES^{\mathrm{sr}-\Environment}|
    \leq
    |\STATES^\Joint|+1
$
considering the totalisation sink.
Thus the ordinary canonical support-restricted environment states are quotients of joint causal states.
The transition structure is similarly also induced from the canonical joint model.

\begin{theorem}[Transitions induced from the canonical joint model]
    \label{th:transitionsEnvironmentInduced}
    Let $\states \in \STATES^{\mathrm{sr}-\Environment}\setminus\{\states^\Environment_\bot\}$ and choose any $s^\Joint\in \psi^{-1}(\states)$.
    Write
    $
        \prob^\Joint(\actions, \observations\mid\states^\Joint)
        \coloneqq
        \Pr((\Actions, \Observations)_{t} = (\actions, \observations)
        \mid
        \States^\Joint_{t} = \states^\Joint),
    $
    and define the marginal $q(\actions \mid \states^\Joint) \coloneqq \sum_{\observations \in \OBSERVATIONS} \prob^\Joint(\actions, \observations \mid \states^\Joint)$.
    For \(\prob^\Joint(\actions, \observations \mid \states^\Joint) > 0\), let \(\delta^\Joint(\states^\Joint, \actions, \observations)\) denote the unique successor joint state given by unifilarity.
    The transitions of the canonical support-restricted environment model are given, independent of the choice of representative $\states^\Joint \in \psi^{-1}(\states)$, by:
    \begin{enumerate}
        \item If $q(\actions \mid \states^\Joint) = 0$,
            $
                T^{\mathrm{sr}-\Environment,(\bot \mid \actions)}_{\states \to \states^\Environment_\bot} = 1
            $
            and
            $T^{\mathrm{sr}-\Environment, (\observations \mid \actions)}_{\states \to \states'} = 0$
            for all $\observations \in \OBSERVATIONS$ and $\states' \in \STATES^{\mathrm{sr}-\Environment}$.
        
        \item If $q(\actions \mid \states^\Joint) > 0$, 
            $
                T^{\mathrm{sr}-\Environment,(\bot \mid \actions)}_{\states \to \states'}
                =
                0
            $
            for all $\states' \in \STATES^{\mathrm{sr}-\Environment}$ and for all $\observations \in \OBSERVATIONS$, 
        \[
            T^{\mathrm{sr}-\Environment,(\observations \mid \actions)}_{\states \to \states'}
            =
            \begin{cases}
                \dfrac{\prob^\Joint(\actions, \observations\mid \states^\Joint)}{q(\actions \mid \states^\Joint)}
                \mathbf 1\!\left\{\states'=\psi\bigl(\delta^\Joint(\states^\Joint, \actions, \observations)\bigr)\right\},
                & \text{if } \prob^\Joint(\actions, \observations\mid \states^\Joint) > 0, \\
                % [2ex]
                0,
                & \text{if } \prob^\Joint(\actions, \observations\mid \states^\Joint)=0.
            \end{cases}
        \]
        \item For the sink state, for all $\actions \in \ACTIONS$,
            $
                T^{\mathrm{sr}-\Environment,(\bot \mid \actions)}_{\states^\Environment_\bot \to \states^\Environment_\bot} = 1.
            $
            and
            $T^{\mathrm{sr}-\Environment, (\observations \mid \actions)}_{\states^\Environment_\bot \to \states'} = 0$
            for all $\observations \in \OBSERVATIONS$ and $\states' \in \STATES^{\mathrm{sr}-\Environment}$.
    \end{enumerate}
\end{theorem}

\begin{proof}
    See~\cref{sec:proofs}.
\end{proof}

At this point the reconstruction has been proved component by component, but not yet stated as a model-level consequence.
\Cref{th:supportRestrictedEnvironmentFactors} gives the non-sink state space from the joint causal states, \cref{th:transitionsEnvironmentInduced} shows that the non-sink transition probabilities are functions of joint-state data, and \cref{prop:sinkState} supplies the remaining sink state and its transitions.
The next theorem packages these pieces into the claim that the canonical joint model determines the entire canonical support-restricted environment model.

\begin{theorem}[Joint reconstruction of the canonical support-restricted environment model]
    \label{th:jointReconstructsSupportRestrictedEnvironment}
    The canonical joint model determines the canonical support-restricted environment model
    $
        (\ACTIONS,\widetilde{\OBSERVATIONS},\STATES^{\mathrm{sr}-\Environment},\mathcal T^{\mathrm{sr}-\Environment})
    $
    of~\cref{def:canonicalSupportRestrictedEnvironmentModel}.
    More explicitly, the interface fixes $\ACTIONS$, the totalisation construction fixes $\widetilde{\OBSERVATIONS}=\OBSERVATIONS\cup\{\bot\}$, \cref{th:supportRestrictedEnvironmentFactors,prop:sinkState} determine the state space, and \cref{th:transitionsEnvironmentInduced,prop:sinkState} determine the transition family.
\end{theorem}

\begin{proof}
    See~\cref{sec:proofs}.
\end{proof}

The agent-side reconstruction is dual: exchanging actions and observations gives the corresponding map from joint causal states to non-sink support-restricted agent states and the induced transition structure.
\begin{table}[ht!]
    \centering
    \small
    \begin{tabular}{p{0.14\textwidth}p{0.38\textwidth}p{0.38\textwidth}}
        \toprule
        & Environment side & Agent side \\
        \midrule
        All future continuations
        &
        \textbf{Canonical environment model}

        Channel: \(\Observations_{:}\mid\Actions_{:}\)

        Predicts future observations under arbitrary future action continuations.
        &
        \textbf{Canonical agent model}

        Channel: \(\Actions_{:}\mid\Observations_{:}\)

        Predicts future actions under arbitrary future observation continuations.
        \\
        \midrule
        Re\-al\-is\-able continuations only
        &
        \textbf{Canonical support-restricted environment model}

        Channel: \(\widetilde{\Observations}_{:}\mid\Actions_{:}\)

        Predicts future observations under action continuations supported by the realised coupling, unsupported action queries emit an environment-side \(\bot\).
        &
        \textbf{Canonical support-restricted agent model}

        Channel: \(\widetilde{\Actions}_{:}\mid\Observations_{:}\) \emph{(dually)}

        Predicts future actions under observation continuations supported by the realised coupling, unsupported observation queries emit an agent-side \(\bot\).
        \\
        \bottomrule
    \end{tabular}
    \caption{
    Environment and agent models along two axes.
    The horizontal axis distinguishes which channel is modelled.
    The vertical axis distinguishes whether predictive equivalence ranges over all future input continuations or only continuations supported by the realised coupling.
    }
    \label{tab:environmentAgentAxes}
\end{table}

The abstract construction above changes the future action continuations against which environment histories are compared.
To see this in practical terms, we return to the running example makes this concrete: the ordinary futures retained by the support-restricted environment model are exactly those compatible with the controller's finite pattern of resets and holds.

\paragraph{Support restriction in the running example.}
\label{ex:supportRestrictedRunning}
Recall (\cref{ex:Example}) that the unrestricted environment model is infinite because arbitrary hold continuations allow unbounded accumulation of evidence about the latent environment state.
The support-restricted construction considers a different object: the environment model induced by the actual coupling to the finite-state controller.
In this controller, action $0$ is realisable in mode $\alpha$, while action $1$ is realisable in modes $\beta$ and $\gamma$.
Thus, querying action $1$ from histories in mode $\alpha$, or querying action $0$ from histories in modes $\beta$ or $\gamma$, is unsupported and is sent to the sink state.
The important mechanism is that the controller bounds hold continuations: after at most two hold actions, it returns to reset.
Therefore the canonical support-restricted environment model never has to represent arbitrarily long hold continuations as ordinary futures.
The realised joint process has five causal states, see~\cref{ex:example}.
The non-sink canonical support-restricted environment states are their images under the factorisation map $\psi$ and are also five, see~\cref{ex:example}.
Together with the totalisation sink $\states^\Environment_\bot$, this gives six canonical support-restricted environment states in the example.

This illustrates~\cref{th:jointReconstructsSupportRestrictedEnvironment}.
The example-specific mechanism is that the controller bounds the supported hold continuations.
The general point is that the ordinary canonical support-restricted environment states are induced from the joint causal states, while unsupported action queries are recorded by the sink.
Thus the same setup can have an infinite unrestricted environment model but a finite canonical support-restricted environment model induced by the realised coupling.

\section{Related work}
\label{sec:related}

Most computational-mechanics treatments of world models in ML focus on the environment channel, giving machine or transducer presentations of environment-side predictive structure~\citep{shaiTransformersRepresentBelief2024, rosasAIVatFundamental2025, boydMonolithsModulesDecomposing2025, riechersNeuralNetworksLeverage2025, riechersNexttokenPretrainingImplies2025, piotrowskiConstrainedBeliefUpdates2025, shaiTransformersLearnFactored2026}.
This connects to observable operator models~\citep{jaegerDiscretetimeDiscretevaluedObservable1998, thonLinksMultiplicityAutomata2015}, predictive state representations~\citep{littmanPredictiveRepresentationsState2001, singhPredictiveStateRepresentations2004}, belief MDPs and bisimulation-style abstractions~\citep{rosasAIVatFundamental2025}, while $\epsilon$-transducers give the canonical minimal unifilar presentation of the environment channel~\citep{barnettComputationalMechanicsInput2015}.
AIXI also provides a Bayesian environment-model perspective for exhaustive planning~\citep{hutterUniversalArtificialIntelligence2005}, and information-theoretic quantities such as empowerment~\citep{klyubin2005empowerment} and directed-information/plasticity constructions~\citep{abelPlasticityMirrorEmpowerment2025} similarly depend on which agent-environment channel is being considered.
Models of the environment are often simply called world models~\citep{haWorldModels2018, dingUnderstandingWorldPredicting2025}.
We however separate these from agent and joint models.
Agent-side predictive models are related to self-predictive universal AI~\citep{cattSelfPredictiveUniversalAI2023}, policy distillation~\citep{rusuPolicyDistillation2016}, plasticity as the dual of empowerment~\citep{abelPlasticityMirrorEmpowerment2025}, and habit-formation models predicting actions from sensorimotor histories~\citep{egbertModelingHabitsSelfsustaining2014, egbertHabitbasedRegulationEssential2014a}.
In partially observable RL, policies can be viewed as maps from histories to actions~\citep{murphyReinforcementLearningOverview2025}, while finite-state controllers give explicit stateful presentations of such policies and are often used to build or interpret recurrent policies~\citep{meuleauLearningFinitestateControllers1999, koulLearningFiniteState2018}; our agent model is the corresponding canonical predictive-state construction for the channel $\Actions_{:}\mid\Observations_{:}$.
Joint models relate to embedded agency~\citep{demskiEmbeddedAgency2020, meulemansEmbeddedUniversalPredictive2025}, coarse-grainings between autonomous systems and open components, Bayesian inversions based on the internal model principle~\citep{baltieriBayesianInterpretationInternal2025, baltieriMathematicalApproachesStudy2025}, and computationally embedded settings such as universal-local environments~\citep{lewandowskiWorldBiggerComputationallyEmbedded2025}.
Finally, support restriction is related to the distinction between arbitrary counterfactual prediction and prediction under policy-, dataset-, controller-, or coupling-induced support.
On the environment side this is analogous to on-/off-policy and offline model learning, where data may come from older policies, experts, teachers, or other sources that the current agent cannot reproduce~\citep{torresanDisentangledRepresentationsCausal2024, rossReductionImitationLearning2011, ostrovskiDifficultyPassiveLearning2021, levineOfflineReinforcementLearning2020}, and to adaptive offline RL outside the original data support~\citep{ghoshOfflineRLPolicies2022}.
It also resonates with $\pi$-bisimulations or on-policy bisimulations~\citep{castroScalableMethodsComputing2020}, which contrast with standard bisimulation abstractions that quantify over all actions rather than only actions realised by a policy.

\section{Conclusions}
\label{sec:conclusions}

In this work we introduced a formal characterisation of world models according to the channel they model: the environment channel, the agent channel, or the realised joint process viewed as a null-input channel.
Using tools from computational mechanics, we defined canonical models that are unique, minimal, and unifilar.
On the environment side, this connects to predictive state representations and belief-MDP-style abstractions, on the agent side, to action-predictive models, policy distillation, and finite-state controllers, and on the joint side, to policy-induced processes and embedded-agency perspectives.

By considering the realised joint process, we then defined canonical support-restricted environment and agent models induced by closed-loop coupling.
These models differ from ordinary canonical environment or agent models because their predictive equivalences are defined relative to continuations supported by the coupling, with unsupported queries handled by totalisation.
On the environment side, we proved that the non-sink canonical support-restricted states factor through the joint causal states, and that the transition structure is induced from the joint model.
The example illustrates that the same finite system can have an infinite unrestricted environment model but a finite canonical support-restricted environment model.

This approach can be further extended to formalise situations when, for instance, an agent considers agents other than itself, such as in behavioural cloning, or joint models of a different agent and a different environment.
An interesting avenue for future work is the application of our definitions to the analysis of trained AI models, particularly when we move past the idealised scenarios of canonical models treated here and consider what \emph{beliefs} can be attributed to an AI system.
This can be achieved, for instance, using mixed-state presentations of both machines~\citep{riechersSpectralSimplicityApparent2018} and transducers~\citep{rosasAIVatFundamental2025}, which would yield agent- and joint-side belief-based counterparts of environment models such as belief MDPs that could be used to identify more structure in, e.g. transformers~\citep{shaiTransformersRepresentBelief2024, riechersNeuralNetworksLeverage2025, riechersNexttokenPretrainingImplies2025, piotrowskiConstrainedBeliefUpdates2025, shaiTransformersLearnFactored2026}.

\section*{Acknowledgments}
Redacted for double blind peer review.

The authors thank Martin Biehl and Scott Garrabrant for inspiring discussions and useful feedback. 
M.B. and F.T. were supported by JST, Moonshot R\&D, Grant Number JPMJMS2012.
M.B. was supported by Advanced Research + Invention Agency (ARIA) through project code MSAI-SE01-P011.

\bibliography{AllEntriesZotero}
\bibliographystyle{abbrv}

\begin{appendices}
\crefalias{section}{appendix}

\appendix
\appendix
%============================================================
\section{Computational mechanics: a brief overview}
%============================================================
\label{sec:compMech}

\subsection{Canonical models}
\subsubsection{$\epsilon$-machines}
Initially, we consider a univariate bi-infinite discrete-time stochastic process of the form $X_: = \{X_t\}_{t \in \Z}$, on a finite alphabet $\mathcal{X}$.
A stochastic process is specified by a consistent family of finite-dimensional distributions
\begin{align}
    \Pr(X_{t:t+L}=x_{t:t+L}),
    \qquad
    x_{t:t+L}\in \mathcal{X}^L,
\end{align}
for all $t \in \Z$ and $L \in \N^+$, where $\mathcal{X}^L$ is the set of words of length $L$.
Equivalently, a consistent family of finite-dimensional distributions specifies probabilities for cylinder events in the space of bi-infinite sequences $\mathcal{X}_{:}$.
For a finite word $x_{t:t+L}\in \mathcal{X}^{L}$, the corresponding cylinder is
\[
    [x_{t:t+L}]
    \coloneqq
    \{x'_{:}\in \mathcal{X}_{:}\mid x'_{t:t+L}=x_{t:t+L}\}.
\]
The finite-dimensional probability $\Pr(X_{t:t+L}=x_{t:t+L})$ is then the probability assigned to this cylinder.
Under the usual consistency conditions, these cylinder probabilities determine a probability measure on $\mathcal{X}_{:}$.

A stochastic process $X_:$ is stationary if its finite-dimensional distributions are invariant under time shifts, i.e.
$
    \Pr(X_{t:t+L} = x_{0:L})
    =
    \Pr(X_{0:L} = x_{0:L})
$
for all $t \in \Z$, $L \in \N^+$, and $x_{0:L}\in \mathcal{X}^L$.
From here onwards, we only consider stationary processes.
A presentation of this process can be given in terms of a \emph{machine}.
\begin{definition}[Machine]
    A machine presentation of $X_:$ is given by a triple $(\mathcal{X}, \mathcal{Z}, \mathcal{T})$ where $\mathcal{Z}$ is a state space and $\mathcal{T}$ are transition dynamics by Markov kernels (or stochastic matrices in the finite case) $\mathcal{T} = \{ T^{(x)} \}_{x \in \mathcal{X}}$ given by:
    \begin{align}
        \mathcal{T} = 
        \{T^{(x)}\}_{x \in \mathcal{X}}, \quad 
        T_{z \to z'}^{(x)} \coloneqq \Pr (Z_{t+1} = z', X_{t} = x \mid Z_{t} = z).
    \end{align}
\end{definition}
When the process is time-independent, we simply write them as $T_{\states \to \states'} = T^{(x)}_{\states \to \states'}$.

Intuitively, computational mechanics gives us a way to construct particularly interesting machine presentations of a process.
Informally, state variables are defined by forming a latent state space for a stationary stochastic system $X_:$ by quotienting histories of $X_:$ into equivalence classes that have identical predictive consequences.
%For a stationary process $\{X_{t}\}_{t\in \Z}$ on a finite alphabet $\mathcal{X}$, d
More formally, we define pasts and futures of $X_:$ as $X_{:t}$ and $X_{t:}$ for all $t \in \Z$ respectively, and an equivalence relation $\sim_\epsilon$ given by:
\begin{align}
    x_{:t} \sim_\epsilon x'_{:t} \iff \Pr(X_{t:} \mid X_{:t} = x_{:t}) = \Pr(X_{t:} \mid X_{:t} = x'_{:t})
\end{align}
The equivalence classes induced by $\sim_\epsilon$ give the so-called \emph{causal states}, forming a set $\STATES$, with associated random variables $\States_{t}$.
This relation can equivalently be seen as induced by a surjective map $\epsilon : \mathcal{X}_{:t} \to \STATES$ given by $\epsilon(x_{:t}) = \states = \{ x'_{:t} \mid x_{:t} \sim_\epsilon x'_{:t} \}$, with states $\states \in \STATES$.
Causal states are minimal sufficient statistics for prediction~\citep{shaliziComputationalMechanicsPattern2001}.
Sufficient means that $X_{t:} \indep X_{:t} \mid \States_{t}$ or $\Pr(X_{t:} \mid X_{:t}) = \Pr(X_{t:} \mid \States_{t})$,
and minimal that, for any other sufficient statistic $Z_{t}$, $Z_{t}$ is a refinement of $\States_{t}$.
The $\epsilon$-map induces dynamics over causal states: given a realised past $x_{:t}$ and the next emitted symbol $x_t$, the extended past $x_{:t+1} = (x_{:t}, x_t)$ determines the next causal state $\epsilon(x_{:t+1})$.
The corresponding transition kernels
% The $\epsilon$-map induces dynamics over pasts, whereby by appending a new observation $x_{t+1}$ to $x_{:t}$ defines a stochastic process with causal states as random variables and transitions between causal states given by Markov kernels 
$\mathcal{T} = \{ T^{(x)} \}_{x \in \mathcal{X}}$ are given by:
\begin{align}
    T^{(x)}_{\states \to \states'} = \Pr(\States_{t+1} = \states', X_{t} = x \mid \States_{t} = \states).
\end{align}
% When the process is time-independent, we simply write them as $T_{\states \to \states'} = T^{(x)}_{\states \to \states'}$.
\begin{definition}[$\epsilon$-machine~\citep{shaliziComputationalMechanicsPattern2001}]
    The $\epsilon$-machine of a stochastic process $X_:$ is the process’s unique, minimal, and unifilar (machine) presentation of the process, given by the tuple $(\mathcal{X}, \STATES, \mathcal{T})$.
\end{definition}
A presentation is unifilar if and only if, for every state $\states \in \STATES$ and symbol $x \in \mathcal{X}$, there is at most one next state $\states'$ with nonzero probability, i.e. $\forall \states, x, |\{\states': T_{\states \states^{\prime}}^{(x)} > 0\}| \leq 1$, such that $H[\States_{t+1} \mid X_{t}, \States_{t}] = 0$.
In other words, a unifilar presentation includes transitions that are deterministic given the current state and the emitted symbol, even though the symbol itself is generally random.

\subsubsection{$\epsilon$-transducers}
\label{sec:epsTransducers}

The above treatment has also been extended to input-output processes, or simply \emph{channels}.
Channels can be seen as defining input-output couplings between stochastic processes.
We consider univariate bi-infinite discrete-time channels, $Y_: \mid X_:$, with input alphabet $\mathcal{X}$ and output alphabet $\mathcal{Y}$.
A channel assigns to each input sequence $x_: \in \mathcal{X}_:$ a probability law over output sequences in $\mathcal{Y}_:$.
We write this family of output laws as
\begin{align}
    Y_: \mid X_: \coloneqq \{ Y_: \mid x_: \}_{x_: \in \mathcal{X}_:},
\end{align}
where $Y_: \mid x_:$ denotes the output stochastic process induced by the fixed input sequence $x_:$.
For each fixed $x_:$, the corresponding finite-dimensional laws are
\begin{align}
    \Pr(Y_{t:t+L} \mid x_:) 
    \coloneqq 
    \{ \Pr(Y_{t:t+L} = y_{t:t+L} \mid x_:) \}_{y_{t:t+L} \in \mathcal{Y}^L},
\end{align}
for all $t \in \Z$ and $L \in \N^+$.
Here $x_:$ is the input sequence supplied to the channel, not necessarily an event of positive probability.
Equivalently, for each fixed $x_:$, the channel specifies a probability measure over output sequences,
\begin{align}
    \Pr(Y_: \mid x_:) 
    \coloneqq 
    \{ \Pr(Y_: \in U \mid x_:) \}_{U \subseteq \mathcal{Y}_:}.
\end{align}
If an input process $X_:$ is also specified, the channel together with the law of $X_:$ induces a joint process $(X,Y)_:$, which can be marginalised to obtain an output process $Y_:$.

A channel $Y_{:} \mid X_{:}$ is stationary if its finite-dimensional laws are invariant under simultaneous shifts of input and output.
That is, for all $t \in \Z$, $L \in \N^+$, $x_{:} \in \mathcal X_{:}$, and $y \in \mathcal Y^L$,
\begin{align}
    \Pr(Y_{t:t+L} = y \mid x_{:})
    =
    \Pr(Y_{0:L} = y \mid (x_{t+z})_{z \in \Z}),
\end{align}
where the right-hand side uses the same input sequence, shifted so that the symbols around time $t$ are treated as the symbols around time $0$.

Stationary channels maps stationary input processes into output stationary processes~\citep{barnettComputationalMechanicsInput2015}.
A channel is causal, or anticipation-free, if future inputs beyond the output horizon do not affect the output law.
Equivalently, for all $t\in \mathbb Z$, $L\in \mathbb N^+$, $y\in \mathcal Y^L$, and input sequences $x_{:},x'_{:}\in \mathcal X_{:}$,
\begin{align}
    x_{:t+L}=x'_{:t+L}
    \implies
    \Pr(Y_{t:t+L}=y\mid x_{:})
    =
    \Pr(Y_{t:t+L}=y\mid x'_{:}).
\end{align}
In this work we focus on stationary causal channels.

We can also give a different presentation of this channel using a \emph{transducer}, as below.
\begin{definition}[Transducer~\citep{rosasAIVatFundamental2025}]
    A transducer presentation of $Y_: \mid X_:$ is given by a quadruple $(\mathcal{X}, \mathcal{Y}, \mathcal{Z}, \mathcal{T})$ where $\mathcal{Z}$ is a state space and $\mathcal{T}$ are transition dynamics by Markov kernels (or stochastic matrices in the finite case) $\mathcal{T} = \{ T^{(y \mid x)} \}_{x \in \mathcal{X}, y \in \mathcal{Y}}$ given by:
    \begin{align}
        T^{(y \mid x)}_{z \to z'} = \Pr(Z_{t+1} = z', Y_{t} = y \mid Z_{t} = z, X_{t} = x).
    \end{align}
\end{definition}

We then define an equivalence relation $\sim_\epsilon$ of pasts of a channel $Y_: \mid X_:$ as:
\begin{align}
    (x, y)_{:t} \sim_\epsilon (x, y)'_{:t} \iff \Pr(Y_{t:} \mid X_{t:}, (X, Y)_{:t} = (x, y)_{:t}) = \Pr(Y_{t:} \mid X_{t:}, (X, Y)_{:t} = (x, y)'_{:t})
\end{align}
which gives a channel's causal states.
This relation can also be seen as induced by a surjective map $\epsilon : (\mathcal{X}, \mathcal{Y})_{:t} \to \STATES$ given by $\epsilon((x, y)_{:t}) = \states = \{ (x, y)'_{:t} \mid (x, y)_{:t} \sim_\epsilon (x, y)'_{:t} \}$.
These causal states are also minimal sufficient statistics for prediction of future outputs~\citep{barnettComputationalMechanicsInput2015}.
The $\epsilon$-map induces dynamics over causal states captured by stochastic matrices $\mathcal{T} = \{ T^{(y \mid x)} \}_{x \in \mathcal{X}, y \in \mathcal{Y}}$ given by:
\begin{align}
    T^{(y \mid x)}_{\states \to \states'} = \Pr(\States_{t+1} = \states', Y_{t} = y \mid \States_{t} = \states, X_{t} = x).
\end{align}

\begin{definition}[$\epsilon$-transducer~\citep{barnettComputationalMechanicsInput2015}]
    The $\epsilon$-transducer of a channel $Y_: \mid X_:$ is the channel’s unique, minimal, unifilar (transducer) presentation of the channel, given by the tuple $(\mathcal{X}, \mathcal{Y}, \STATES, \mathcal{T})$.
\end{definition}

\section{Proofs}
\label{sec:proofs}

\begin{proof}[Proof of~\cref{prop:sinkState}]
    If $\widetilde \historiesfiltering_{:t}$ contains an occurrence of $\bot$, then by the definition of the channel, $\Pr\bigl(\widetilde \Observations_{t:}=\bot^\infty \mid \Actions_{t:}, \widetilde \HistoriesFiltering_{:t}=\widetilde \historiesfiltering_{:t}\bigr)=1$.
    Hence any two such histories induce the same future morph (i.e. conditional distribution of futures) and therefore lie in the same causal state. 
    The displayed transition identity follows immediately from the same observation.
\end{proof}

\begin{proof}[Proof of~\cref{th:supportRestrictedEnvironmentFactors}]
    Let $\historiesfiltering_{:t},\historiesfiltering'_{:t}\in\HISTORIESFILTERING_{:t}$ be ordinary histories such that
    $
        \epsilon_\Joint(\historiesfiltering_{:t})
        =
        \epsilon_\Joint(\historiesfiltering'_{:t})
    $.
    By joint causal equivalence (\cref{eq:jointCausalStates}), for every $L\geq 1$ and every ordinary word $(\actions,\observations)_{t:t+L-1}\in(\ACTIONS\times\OBSERVATIONS)^L$,
    \begin{align}
        &\Pr\bigl(
        (\Actions,\Observations)_{t:t+L-1}
        =
        (\actions,\observations)_{t:t+L-1}
        \mid
        \HistoriesFiltering_{:t}=\historiesfiltering_{:t}\bigr)
        \nonumber\\
        &\quad =
        \Pr\bigl(
        (\Actions,\Observations)_{t:t+L-1}
        =
        (\actions,\observations)_{t:t+L-1}
        \mid
        \HistoriesFiltering_{:t}=\historiesfiltering'_{:t}\bigr).
    \end{align}
    We now show that this joint-cylinder equality implies equality of the finite-dimensional laws of the totalised support-restricted environment channel.
    Specifically, we prove by induction on $L$ that, for, every input word $\actions_{t:t+L-1}\in\ACTIONS^L$, and every output word $\widetilde{\observations}_{t:t+L-1}\in\widetilde{\OBSERVATIONS}^L$, the previous equality implies the following
    \begin{align}
        &\Pr\bigl(\widetilde\Observations_{t:t+L-1}=\widetilde{\observations}_{t:t+L-1}
        \mid
        \Actions_{t:t+L-1}=\actions_{t:t+L-1},
        \widetilde\HistoriesFiltering_{:t}=\historiesfiltering_{:t}\bigr)
        \nonumber\\
        &\quad =
        \Pr\bigl(\widetilde\Observations_{t:t+L-1}=\widetilde{\observations}_{t:t+L-1}
        \mid
        \Actions_{t:t+L-1}=\actions_{t:t+L-1},
        \widetilde\HistoriesFiltering_{:t}=\historiesfiltering'_{:t}\bigr).
        \label{eq:refactoredSupportRestrictedCylinder}
    \end{align}
    Since the alphabets are finite, equality of all finite cylinders then implies equality of the full support-restricted future morphs.

    For $L=1$, fix $\actions\in\ACTIONS$ and write
    \begin{align}
        q_{\historiesfiltering}(\actions)
        \coloneqq
        \sum_{\observations\in\OBSERVATIONS}
        \Pr\bigl((\Actions,\Observations)_t=(\actions,\observations)
        \mid
        \HistoriesFiltering_{:t}=\historiesfiltering_{:t}\bigr).
    \end{align}
    Joint causal equivalence gives equality of the corresponding one-step joint probabilities, and hence
    $
        q_{\historiesfiltering}(\actions)
        =
        q_{\historiesfiltering'}(\actions)
    $.
    If this common value is zero, then $\actions$ is unsupported at both histories and the totalised channel emits $\bot$ with probability one at both histories.
    If it is positive, then $\actions$ is supported at both histories and, for every $\observations\in\OBSERVATIONS$,
    \begin{align}
        \Pr\bigl(\widetilde\Observations_t=\observations
        \mid
        \Actions_t=\actions,
        \widetilde\HistoriesFiltering_{:t}=\historiesfiltering_{:t}\bigr)
        =
        \frac{
        \Pr\bigl((\Actions,\Observations)_t=(\actions,\observations)
        \mid
        \HistoriesFiltering_{:t}=\historiesfiltering_{:t}\bigr)}
        {q_{\historiesfiltering}(\actions)}.
    \end{align}
    The same formula holds with $\historiesfiltering'_{:t}$ in place of $\historiesfiltering_{:t}$, with the same numerator and denominator.
    In this supported case both histories emit $\bot$ with probability zero.
    This proves~\cref{eq:refactoredSupportRestrictedCylinder} for $L=1$.

    Now assume the claim holds for length $L$ for every pair of ordinary histories with the same joint causal state, and consider a word of length $L+1$ from the fixed histories above.
    If the first output symbol is $\bot$, then by totalisation all later output symbols must be $\bot$; otherwise the word has probability zero from both histories.
    If all later output symbols are $\bot$, the probability of the whole word is the common one-step probability of emitting $\bot$ under the first queried action, since after the first $\bot$ the sink emits $\bot$ with probability one.

    It remains to consider the case where the first output symbol is an ordinary observation $\observations_t\in\OBSERVATIONS$.
    If the one-step probability of the pair $(\actions_t,\observations_t)$ is zero, then the whole word has probability zero from both histories.
    Otherwise the same pair is realisable from both histories.
    Conditioning the equal joint future laws on this common positive-probability prefix gives
    \begin{align}
        \epsilon_\Joint(\historiesfiltering_{:t},\actions_t,\observations_t)
        =
        \epsilon_\Joint(\historiesfiltering'_{:t},\actions_t,\observations_t).
    \end{align}
    The induction hypothesis applied to these extended ordinary histories gives equality of the conditional probabilities of the remaining output word $\widetilde{\observations}_{t+1:t+L}$ under the remaining input word $\actions_{t+1:t+L}$.
    By the chain rule, the probability of the whole length-$(L+1)$ word is the one-step probability of emitting $\observations_t$ under input $\actions_t$ multiplied by the corresponding remaining-word probability.
    The one-step factors are equal by the base case, and the remaining-word factors are equal by the induction hypothesis, so~\cref{eq:refactoredSupportRestrictedCylinder} holds for length $L+1$.

    This allows us to define a map from joint causal states to non-sink support-restricted environment states.
    Given a joint causal state $\states^\Joint$, choose any ordinary history $\historiesfiltering_{:t}$ such that $\epsilon_\Joint(\historiesfiltering_{:t})=\states^\Joint$, and set
    $
        \psi(\states^\Joint)
        \coloneqq
        \epsilon_{\mathrm{sr}-\Environment}(\historiesfiltering_{:t}).
    $
    This definition is independent of the chosen representative: if another ordinary history has the same joint causal state, the result just proved implies that it has the same support-restricted environment causal state.
    Thus $\psi$ is well defined.
    Moreover, every non-sink support-restricted environment state is the state of some ordinary history, and that ordinary history also has a joint causal state.
    Therefore every non-sink support-restricted environment state lies in the image of $\psi$, so $\psi$ is surjective.
    Finally, the factorising map is unique.
    Suppose that $\psi'$ is another map satisfying
    $
        \epsilon_{\mathrm{sr}-\Environment}
        =
        \psi'\circ\epsilon_\Joint
    $
    on ordinary histories.
    Let $\states^\Joint\in\STATES^\Joint$.
    Since $\epsilon_\Joint$ is surjective, there is an ordinary history $\historiesfiltering_{:t}$ with
    $
        \epsilon_\Joint(\historiesfiltering_{:t})=\states^\Joint.
    $
    Then
    $
        \psi'(\states^\Joint)
        =
        \psi'(\epsilon_\Joint(\historiesfiltering_{:t}))
        =
        \epsilon_{\mathrm{sr}-\Environment}(\historiesfiltering_{:t})
        =
        \psi(\states^\Joint).
    $
    Since this holds for every $\states^\Joint$, we have $\psi'=\psi$.
\end{proof}

\begin{proof}[Proof of~\cref{th:transitionsEnvironmentInduced}]
    The first point to check before reading the formulas from the channel definition is that they do not depend on the representative $\states^\Joint\in\psi^{-1}(\states)$.
    Let $\states_1^\Joint,\states_2^\Joint\in\psi^{-1}(\states)$, and choose ordinary histories $\historiesfiltering^1_{:t},\historiesfiltering^2_{:t}$ with
    $
        \epsilon_\Joint(\historiesfiltering^i_{:t})=\states_i^\Joint
    $
    for $i\in\{1,2\}$.
    Since both joint states map to $\states$, these histories have the same support-restricted environment causal state.
    Therefore the one-step output laws of the totalised channel $\widetilde\Observations_:\mid\Actions_:$ agree from the two histories under every queried action.

    Write
    \begin{align}
        p_i(\actions,\observations)
        \coloneqq
        \prob^\Joint(\actions,\observations\mid\states_i^\Joint),
        \qquad
        q_i(\actions)
        \coloneqq
        \sum_{\observations\in\OBSERVATIONS}
        p_i(\actions,\observations).
    \end{align}
    By~\cref{def:realisableActions,def:supportRestrictedEnv}, $q_i(\actions)=0$ exactly when the queried action $\actions$ is unsupported from $\historiesfiltering^i_{:t}$, in which case the totalised channel emits $\bot$ with probability one.
    Since the one-step output laws agree, this unsupported case occurs for $i=1$ if and only if it occurs for $i=2$.
    Equivalently,
    \begin{align}
        q_1(\actions)>0
        \iff
        q_2(\actions)>0.
    \end{align}
    In the supported case, the same channel definition gives the ordinary output probabilities by normalising the joint one-step law:
    \begin{align}
        \Pr\bigl(\widetilde\Observations_t=\observations
        \mid
        \Actions_t=\actions,
        \widetilde\HistoriesFiltering_{:t}=\historiesfiltering^i_{:t}
        \bigr)
        =
        \frac{p_i(\actions,\observations)}{q_i(\actions)}.
    \end{align}
    Equality of the one-step output laws therefore gives, for every $\observations\in\OBSERVATIONS$,
    \begin{align}
        \frac{p_1(\actions,\observations)}{q_1(\actions)}
        =
        \frac{p_2(\actions,\observations)}{q_2(\actions)}.
    \end{align}
    In particular, $p_1(\actions,\observations)>0$ if and only if $p_2(\actions,\observations)>0$.

    It remains to be checked that the successor state in the supported ordinary case also descends to the quotient.
    Suppose $p_1(\actions,\observations)>0$, equivalently $p_2(\actions,\observations)>0$, and let
    $
        \historiesfiltering^i_{:t+1}
        \coloneqq
        (\historiesfiltering^i_{:t},\actions,\observations).
    $
    Conditioning the equal future morphs on this common positive-probability prefix gives
    \begin{align}
        \epsilon_{\mathrm{sr}-\Environment}(\historiesfiltering^1_{:t+1})
        =
        \epsilon_{\mathrm{sr}-\Environment}(\historiesfiltering^2_{:t+1}).
    \end{align}
    By unifilarity of the canonical joint model,
    \begin{align}
        \epsilon_\Joint(\historiesfiltering^i_{:t+1})
        =
        \delta^\Joint(\states_i^\Joint,\actions,\observations).
    \end{align}
    Applying $\psi$ gives
    \begin{align}
        \psi\bigl(\delta^\Joint(\states_1^\Joint,\actions,\observations)\bigr)
        =
        \psi\bigl(\delta^\Joint(\states_2^\Joint,\actions,\observations)\bigr).
    \end{align}
    Thus the branch condition, the ordinary output probabilities, and the successor support-restricted state are independent of the chosen representative.

    The transition formula now follows directly from~\cref{def:supportRestrictedEnv}.
    If $q(\actions \mid \states^\Joint)=0$, the queried action is unsupported, so the channel emits $\bot$ and moves to the sink with probability one.
    If $q(\actions \mid \states^\Joint)>0$, the channel emits an ordinary observation $\observations$ with probability
    \begin{align}
        \Pr\bigl(\widetilde \Observations_{t}=\observations
        \mid
        \Actions_{t}=\actions,
        \States_{t}^{\mathrm{sr}-\Environment}=\states\bigr)
        =
        \frac{\prob^\Joint(\actions,\observations\mid\states^\Joint)}
        {q(\actions\mid\states^\Joint)}.
    \end{align}
    If $\prob^\Joint(\actions,\observations\mid\states^\Joint)=0$, the corresponding transition has probability zero; otherwise the successor state is $\psi\bigl(\delta^\Joint(\states^\Joint,\actions,\observations)\bigr)$.
    The sink-state transitions are those of~\cref{prop:sinkState}.
\end{proof}

\begin{proof}[Proof of~\cref{th:jointReconstructsSupportRestrictedEnvironment}]
    The input alphabet is the original action alphabet $\ACTIONS$ of the joint interface.
    The output alphabet is the totalised observation alphabet $\widetilde{\OBSERVATIONS}=\OBSERVATIONS\cup\{\bot\}$ introduced in~\cref{def:supportRestrictedEnv}.
    
    For the state space, \cref{th:supportRestrictedEnvironmentFactors} identifies the non-sink states as the quotient of joint causal states given by the map
    $
        \psi:\STATES^\Joint\to\STATES^{\mathrm{sr}-\Environment}\setminus\{\states^\Environment_\bot\}.
    $
    \Cref{prop:sinkState} supplies the remaining state, the absorbing totalisation sink $\states^\Environment_\bot$.
    Since every extended history is either ordinary or contains an occurrence of $\bot$, these two parts give the whole state space
    $
        \STATES^{\mathrm{sr}-\Environment}
        =
        \bigl(\STATES^{\mathrm{sr}-\Environment}\setminus\{\states^\Environment_\bot\}\bigr)
        \cup
        \{\states^\Environment_\bot\}.
    $

    It remains only to specify the transition family $\mathcal T^{\mathrm{sr}-\Environment}$.
    For non-sink states, \cref{th:transitionsEnvironmentInduced} gives the transition probabilities for every queried action and every output in $\widetilde{\OBSERVATIONS}$, and proves that these probabilities are independent of the chosen joint-state representative.
    For the sink state, \cref{prop:sinkState} gives the absorbing transitions:
    every queried action emits $\bot$ and remains at $\states^\Environment_\bot$ with probability one.
    Thus all entries of $\mathcal T^{\mathrm{sr}-\Environment}$ are determined.

    These four components are exactly the $\epsilon$-transducer
    $
        (\ACTIONS,\widetilde{\OBSERVATIONS},\STATES^{\mathrm{sr}-\Environment},\mathcal T^{\mathrm{sr}-\Environment})
    $
    of~\cref{def:canonicalSupportRestrictedEnvironmentModel}.
\end{proof}

\section{Running example}
\label{ex:example}

We spell out the running binary POMDP/controller example used in the main text.
The example separates the unrestricted environment model, the agent model, the realised joint model, and the canonical support-restricted environment model.
It also illustrates why the canonical support-restricted model induced by the realised coupling can be finite even when the unrestricted environment model is infinite.

\begin{table}[ht!]
    \centering
    \small
    \begin{tabular}{p{0.27\linewidth}p{0.12\linewidth}p{0.48\linewidth}}
        \toprule
        Model & Number of causal states & Reason \\
        \midrule
        Canonical unrestricted environment model
        & $\infty$
        & Arbitrarily long counterfactual hold continuations generate infinitely many posterior beliefs over the latent environment state, and these beliefs are predictively distinct. \\
        Canonical agent model
        & $3$
        & The deterministic controller has three pairwise predictively distinct memory modes, $\alpha, \beta, \gamma$. \\
        Canonical joint model
        & $5$
        & The realised coupling reaches five predictive configurations:
        $\states^\Joint_{\alpha}$,
        $\states^\Joint_{\beta 1}$,
        $\states^\Joint_{\beta 0}$,
        $\states^\Joint_{\gamma 11}$,
        $\states^\Joint_{\gamma 01}$. \\
        Canonical support-restricted environment model
        & $6$
        & The five ordinary states are the images of the joint states under $\psi$, together with the totalisation sink $\states^\Environment_\bot$. \\
        Canonical support-restricted agent model
        & $4$
        & The environment has full observation support in this example, so support restriction does not further merge the controller's three action-predictive modes (+ 1 totalisation sink $\states^\Agent_\bot$). \\
        \bottomrule
    \end{tabular}
    \caption{
        State-complexity summary for the running POMDP/controller example.
        The unrestricted environment model is infinite, while the realised joint model and the induced canonical support-restricted environment model are finite.
    }
    \label{tab:exampleStateComplexities}
\end{table}

\subsection{The setup: environment and agent}

We consider a coupled system with binary actions and observations,
\[
    \ACTIONS=\OBSERVATIONS=\{0,1\}.
\]
The environment has a hidden binary state $\Environment_t\in \{0,1\}$.
The agent is a deterministic finite-state controller with memory state $M_t\in \{\alpha, \beta, \gamma\}$.
At each time $t$, the agent emits $\Actions_t$ as a function of $M_t$, the environment updates $\Environment_{t+1}$ according to $\Actions_t$, the environment emits $\Observations_t$, and the controller updates $M_{t+1}$ from $(M_t, \Observations_t)$.
The observed joint symbol is
\[
    W_t\coloneqq(\Actions_t, \Observations_t)\in \ACTIONS\times\OBSERVATIONS.
\]

For the environment, fix a sensor noise parameter $\eta\in(0, \tfrac12)$.
The controlled transition is
\begin{align}
    \label{eq:env-transition}
    \Pr(\Environment_{t+1}=1\mid \Environment_t=e, \Actions_t=a)
    =
    \begin{cases}
        \tfrac12, & a=0, \\
        e, & a=1.
    \end{cases}
\end{align}
Thus action $0$ resets the hidden state to a fair coin, while action $1$ holds the hidden state fixed.
The observation is a noisy readout of the new hidden state:
\begin{align}
    \label{eq:env-emission}
    \Pr(\Observations_t=1\mid \Environment_{t+1}=1)=1-\eta,
    \qquad
    \Pr(\Observations_t=1\mid \Environment_{t+1}=0)=\eta.
\end{align}

The controller emits actions according to
\begin{align}
    \label{eq:agent-action}
    \Actions_t =
    \begin{cases}
        0, & M_t=\alpha, \\
        1, & M_t\in \{\beta, \gamma\}.
    \end{cases}
\end{align}
Its memory update is
\begin{align}
    \label{eq:agent-update}
    M_{t+1}
    =
    \begin{cases}
        \beta, & M_t=\alpha, \\
        \gamma, & M_t=\beta \text{ and } \Observations_t=1, \\
        \alpha, & M_t=\beta \text{ and } \Observations_t=0, \\
        \alpha, & M_t=\gamma.
    \end{cases}
\end{align}
Thus the controller always resets in mode $\alpha$, then holds in mode $\beta$.
If the observation after this hold is $1$, it holds once more in mode $\gamma$ before returning to reset.
If the observation after the first hold is $0$, it resets immediately.

\subsection{Canonical models}
\subsubsection{Canonical environment model}

Let
\[
    \prob_t
    \coloneqq
    \Pr(\Environment_t=1\mid \HistoriesFiltering_{:t}=h_{:t})
\]
be the filtering belief after an action-observation history $h_{:t}$.
For a queried action $a\in \ACTIONS$, define the predictive prior
\[
    q_t(a)
    \coloneqq
    \Pr(\Environment_{t+1}=1\mid \HistoriesFiltering_{:t}=h_{:t}, \Actions_t=a).
\]
By~\cref{eq:env-transition},
\begin{align}
    \label{eq:predictive-prior}
    q_t(0)=\tfrac12,
    \qquad
    q_t(1)=\prob_t.
\end{align}

After observing $\Observations_t=o$, Bayes' rule gives
\begin{align}
    \label{eq:belief-update-one}
    \prob_{t+1}
    =
    \frac{(1-\eta)q_t(a)}
    {(1-\eta)q_t(a)+\eta(1-q_t(a))}
    \qquad
    \text{if }o=1,
\end{align}
and
\begin{align}
    \label{eq:belief-update-zero}
    \prob_{t+1}
    =
    \frac{\eta q_t(a)}
    {\eta q_t(a)+(1-\eta)(1-q_t(a))}
    \qquad
    \text{if }o=0.
\end{align}
Under the hold action $a=1$, the one-step predictive probability of observing $1$ is
\begin{align}
    \label{eq:one-step-predictive}
    \Pr(\Observations_t=1\mid \HistoriesFiltering_{:t}=h_{:t}, \Actions_t=1)
    &=
    (1-\eta)\prob_t+\eta(1-\prob_t) \nonumber\\
    &=
    \eta+\prob_t(1-2\eta).
\end{align}
Since $\eta\in(0, \tfrac12)$, the map $p\mapsto\eta+p(1-2\eta)$ is strictly increasing.
Therefore two histories with different beliefs $\prob_t\neq\prob'_t$ are distinguished by the counterfactual one-step input $\Actions_t=1$, and so cannot be equivalent environment causal states.

It remains to see that infinitely many such beliefs are reachable under arbitrary future action inputs.
Let
\[
    L_t\coloneqq \frac{\prob_t}{1-\prob_t},
    \qquad
    r\coloneqq \frac{1-\eta}{\eta}>1.
\]
Under repeated hold actions, the odds update as
\[
    L_{t+1}
    =
    \begin{cases}
        L_t r, & \Observations_t=1, \\
        L_t r^{-1}, & \Observations_t=0.
    \end{cases}
\]
Starting from $\prob_0=\tfrac12$, so $L_0=1$, after $n$ hold steps with $k$ observations equal to $1$,
\begin{align}
    \label{eq:belief-after-n}
    L_n=r^{2k-n},
    \qquad
    \prob_n
    =
    \frac{r^{2k-n}}{1+r^{2k-n}}
    =
    \frac{(1-\eta)^k\eta^{n-k}}
    {(1-\eta)^k\eta^{n-k}+\eta^k(1-\eta)^{n-k}}.
\end{align}
For example, taking $k=n$ gives the infinite sequence $L_n=r^n$.
Hence there are infinitely many distinct posterior beliefs, and each gives a distinct environment causal state by~\cref{eq:one-step-predictive}.
Thus the unrestricted canonical environment model has infinitely many causal states.

\subsubsection{Canonical agent model}

Viewed as a channel $\Actions_{:} \mid \Observations_{:}$, the deterministic controller has three causal states.
Each memory mode $M_t \in \{\alpha, \beta, \gamma\}$ determines the future action stream under any future observation continuation.
The three modes are pairwise predictively distinct.

Mode \(\alpha\) differs from \(\beta\) and \(\gamma\) already at the current action: \(\alpha\) emits \(0\), while \(\beta\) and \(\gamma\) emit \(1\).
Modes \(\beta\) and \(\gamma\) are distinct because under the observation \(\Observations_t = 1\), \(\beta\) transitions to \(\gamma\), so \(\Actions_{t+1} = 1\), whereas \(\gamma\) transitions to \(\alpha\), so \(\Actions_{t+1}=0\).
% Mode $\alpha$ differs from $\beta$ and $\gamma$: under any current observation, $\alpha$ transitions to $\beta$, so the next action is $\Actions_{t+1}=1$, while $\gamma$ transitions to $\alpha$, so the next action is $\Actions_{t+1}=0$; and $\beta$ differs from $\alpha$ under the observation $\Observations_t=0$, since $\beta$ then transitions to $\alpha$.
% Modes $\beta$ and $\gamma$ are also distinct: under the observation $\Observations_t=1$, $\beta$ transitions to $\gamma$, so the next action is $\Actions_{t+1}=1$, whereas $\gamma$ transitions to $\alpha$, so the next action is $\Actions_{t+1}=0$.
Therefore
\[
    \STATES^\Agent=\{\alpha, \beta, \gamma\},
    \qquad
    |\STATES^\Agent|=3.
\]

\subsubsection{Canonical joint model}

For the joint model, we consider the realised coupled process $W_t = (\Actions_t, \Observations_t)$.
The process is generally not Markov on the emitted symbols alone, but it has a finite predictive presentation.
The relevant predictive configurations are determined by the controller mode together with the belief values that can occur under the realised coupling.

After a reset from mode $\alpha$, the latent state is a fair coin, so the next posterior is $1-\eta$ after observing $1$ and $\eta$ after observing $0$.
From mode $\beta$, the controller holds once.
Starting from the two beliefs $\eta$ and $1-\eta$, if the observation is $1$, the posterior becomes either
\[
    p^{(11)}
    \coloneqq
    \frac{(1-\eta)^2}{(1-\eta)^2+\eta^2}
\]
from prior $1-\eta$, or $\tfrac12$ from prior $\eta$.
If the observation is $0$, the controller returns to $\alpha$, where the next reset makes the precise posterior irrelevant for future prediction.
Mode $\gamma$ performs one further hold and then returns to $\alpha$.

Thus the realised joint process has the following five predictive states:
\[
    \STATES^\Joint
    =
    \{
        \states^\Joint_\alpha,
        \states^\Joint_{\beta 1},
        \states^\Joint_{\beta 0},
        \states^\Joint_{\gamma 11},
        \states^\Joint_{\gamma 01}
    \},
\]
where
\[
\begin{array}{rcl}
    \states^\Joint_\alpha &:& M_t=\alpha, \\
    \states^\Joint_{\beta 1} &:& M_t=\beta, \ \prob_t=1-\eta, \\
    \states^\Joint_{\beta 0} &:& M_t=\beta, \ \prob_t=\eta, \\
    \states^\Joint_{\gamma 11} &:& M_t=\gamma, \ \prob_t=p^{(11)}, \\
    \states^\Joint_{\gamma 01} &:& M_t=\gamma, \ \prob_t=\tfrac12.
\end{array}
\]

Let
\[
    c\coloneqq (1-\eta)^2+\eta^2,
    \qquad
    d\coloneqq 2\eta(1-\eta),
    \qquad
    g\coloneqq \eta+p^{(11)}(1-2\eta).
\]
The nonzero transitions of the joint presentation are:
\[
\begin{array}{rclcrcl}
    \states^\Joint_\alpha
    &\xrightarrow{(0,1)}&
    \states^\Joint_{\beta 1}
    &\text{with probability}&
    \tfrac12,
    \qquad&
    \states^\Joint_\alpha
    &\xrightarrow{(0,0)}
    \states^\Joint_{\beta 0}
    \text{ with probability }\tfrac12, \\[0.5ex]
    \states^\Joint_{\beta 1}
    &\xrightarrow{(1,1)}&
    \states^\Joint_{\gamma 11}
    &\text{with probability}&
    c,
    \qquad&
    \states^\Joint_{\beta 1}
    &\xrightarrow{(1,0)}
    \states^\Joint_{\alpha}
    \text{ with probability }d, \\[0.5ex]
    \states^\Joint_{\beta 0}
    &\xrightarrow{(1,1)}&
    \states^\Joint_{\gamma 01}
    &\text{with probability}&
    d,
    \qquad&
    \states^\Joint_{\beta 0}
    &\xrightarrow{(1,0)}
    \states^\Joint_{\alpha}
    \text{ with probability }c, \\[0.5ex]
    \states^\Joint_{\gamma 11}
    &\xrightarrow{(1,1)}&
    \states^\Joint_{\alpha}
    &\text{with probability}&
    g,
    \qquad&
    \states^\Joint_{\gamma 11}
    &\xrightarrow{(1,0)}
    \states^\Joint_{\alpha}
    \text{ with probability }1-g, \\[0.5ex]
    \states^\Joint_{\gamma 01}
    &\xrightarrow{(1,1)}&
    \states^\Joint_{\alpha}
    &\text{with probability}&
    \tfrac12,
    \qquad&
    \states^\Joint_{\gamma 01}
    &\xrightarrow{(1,0)}
    \states^\Joint_{\alpha}
    \text{ with probability }\tfrac12.
\end{array}
\]
This presentation is unifilar: given the current state and emitted symbol $(\Actions_t, \Observations_t)$, the next state is determined.
The five states are pairwise predictively distinct.
The state $\states^\Joint_\alpha$ is distinguished from all others because it emits action $0$ rather than action $1$.
The states $\states^\Joint_{\beta 1}$ and $\states^\Joint_{\beta 0}$ have different probabilities of observing $1$.
The states $\states^\Joint_{\gamma 11}$ and $\states^\Joint_{\gamma 01}$ also have different probabilities of observing $1$.
Finally, $\beta$-states cannot merge with $\gamma$-states because their successor structure differs: from a $\beta$-state one may transition to a $\gamma$-state, while from a $\gamma$-state one returns to $\alpha$.
Therefore the canonical joint model has exactly five causal states.

\subsection{Canonical support-restricted models}
\subsubsection{Canonical support-restricted environment model}

The unrestricted environment model is infinite because it must distinguish predictions under arbitrary future action continuations, including arbitrarily long hold sequences.
Under the realised controller, however, arbitrarily long hold continuations are not supported.
The controller emits action $0$ in mode $\alpha$, action $1$ in modes $\beta$ and $\gamma$, and then returns to $\alpha$ after at most two consecutive holds.

By the general construction in~\cref{sec:supportRestricted}, each ordinary canonical support-restricted environment state is the image under $\psi$ of a joint causal state.
In this example, the five non-sink images are pairwise distinct:
\[
\begin{array}{rcl}
    \states_\alpha^{\mathrm{sr}-\Environment} &\coloneqq& \psi(\states^\Joint_\alpha), \\
    \states_{\beta 1}^{\mathrm{sr}-\Environment} &\coloneqq& \psi(\states^\Joint_{\beta 1}), \\
    \states_{\beta 0}^{\mathrm{sr}-\Environment} &\coloneqq& \psi(\states^\Joint_{\beta 0}), \\
    \states_{\gamma 11}^{\mathrm{sr}-\Environment} &\coloneqq& \psi(\states^\Joint_{\gamma 11}), \\
    \states_{\gamma 01}^{\mathrm{sr}-\Environment} &\coloneqq& \psi(\states^\Joint_{\gamma 01}).
\end{array}
\]
Together with the totalisation sink $\states^\Environment_\bot$, the canonical support-restricted environment state space is
\[
    \STATES^{\mathrm{sr}-\Environment}
    =
    \{
        \states_\alpha^{\mathrm{sr}-\Environment},
        \states_{\beta 1}^{\mathrm{sr}-\Environment},
        \states_{\beta 0}^{\mathrm{sr}-\Environment},
        \states_{\gamma 11}^{\mathrm{sr}-\Environment},
        \states_{\gamma 01}^{\mathrm{sr}-\Environment},
        \states^\Environment_\bot
    \}.
\]

The supported action sets are
\[
    \ACTIONS_1(\states_\alpha^{\mathrm{sr}-\Environment})=\{0\},
    \qquad
    \ACTIONS_1(\states_{\beta 1}^{\mathrm{sr}-\Environment})
    =
    \ACTIONS_1(\states_{\beta 0}^{\mathrm{sr}-\Environment})
    =
    \ACTIONS_1(\states_{\gamma 11}^{\mathrm{sr}-\Environment})
    =
    \ACTIONS_1(\states_{\gamma 01}^{\mathrm{sr}-\Environment})
    =
    \{1\}.
\]
Unsupported queried actions are sent to $\states^\Environment_\bot$, and from $\states^\Environment_\bot$ every queried action emits $\bot$ and remains in $\states^\Environment_\bot$.
The ordinary transitions are obtained from the joint transitions above by applying~\cref{th:transitionsEnvironmentInduced}.

The six states are pairwise distinct.
The sink state is distinct from every ordinary state because it emits $\bot$ forever.
The state $\states_\alpha^{\mathrm{sr}-\Environment}$ is distinct from every other ordinary state because querying action $1$ immediately yields $\bot$, whereas action $1$ is supported from the other ordinary states.
The two $\beta$-states are distinct because under action $1$ they assign different probabilities to $\Observations_t=1$, namely $c$ and $d$.
The two $\gamma$-states are distinct because under action $1$ they assign different probabilities to $\Observations_t=1$, namely $g$ and $\tfrac12$.
No $\beta$-state can merge with any $\gamma$-state: under the queried action word $11$, a $\beta$-state has positive probability of producing two ordinary observations before any occurrence of $\bot$, while from a $\gamma$-state the first queried hold returns the controller to $\alpha$, so the second queried hold emits $\bot$ almost surely.
Therefore the canonical support-restricted environment model has exactly six causal states.

\subsubsection{Canonical support-restricted agent model}

For the agent side, support restriction would restrict future observation continuations to those supported by the realised environment.
In this example, the noisy observation kernel has full support because $\eta\in(0, \tfrac12)$.
Consequently, every finite observation continuation has positive probability under the coupled process.
The canonical support-restricted agent model therefore has the same three causal states as the unrestricted canonical agent model, together with a sink state $\states^\Agent_\bot$:
$
    \STATES^{\mathrm{sr}-\Agent}
    =
    \STATES^\Agent \cup \states^\Agent_\bot.
    % =
    % \{\alpha, \beta, \gamma\}.
$
This is specific to the present example. 
In general, environmental support restrictions can merge or alter the canonical support-restricted agent states.

\section{Relation between the three canonical models, with extra structure}
\label{sec:extraStructure}

Suppose that the environment and agent models are not only given as separate canonical channels, but as compatible components of a factorised closed-loop presentation.
Concretely, assume that:
\begin{itemize}
    \item the environment model has causal state space $\STATES^\Environment$ and the agent model has causal state space $\STATES^\Agent$, with compatible alphabets and timing conventions,
    \item closing the two channels defines a stationary joint process on $(\Actions, \Observations)_{:}$,
    \item the product state
        $
            Z_{t} \coloneqq (\States^\Environment_{t}, \States^\Agent_{t})
        $
        is sufficient for predicting the realised joint future, but is \emph{not} a causal state of the joint process in general.
        Equivalently, assume the conditional independence
        \begin{align}
            (\Actions, \Observations)_{t:}
            \indep
            \HistoriesFiltering_{:t}
            \mid
            Z_{t},
        \end{align}
        or, in probabilistic form,
        $
            \Pr((\Actions, \Observations)_{t:} \mid \HistoriesFiltering_{:t})
            =
            \Pr((\Actions, \Observations)_{t:}\mid Z_{t}),
        $
    \item the reachable product states, i.e. the states in the support of $Z_{t}$:
        \begin{align}
            \mathrm{Reach}(\STATES^\Environment \times \STATES^\Agent)
            \coloneqq
            \{(\states^\Environment, \states^\Agent) \in
            \STATES^\Environment \times \STATES^\Agent
            \mid
            \Pr(Z_{t} = (\states^\Environment, \states^\Agent)) > 0
            \},
        \end{align}
        carry well-defined transition dynamics and form a presentation of the realised joint process.
\end{itemize}

Under these assumptions, the reachable subset $\mathrm{Reach}(\STATES^\Environment\times\STATES^\Agent)$ is a sufficient presentation of the joint process.
Since the canonical joint model is minimal among sufficient predictive presentations, its causal states are a quotient of this reachable product presentation.
Equivalently, there is a surjective map
$
    \chi:\mathrm{Reach}(\STATES^\Environment\times\STATES^\Agent)\to\STATES^\Joint
$
such that, for every history $\historiesfiltering_{:t}$ inducing the reachable product state
$
    z(\historiesfiltering_{:t})
    \coloneqq
    \bigl(
        \epsilon_\Environment(\historiesfiltering_{:t}),
        \epsilon_\Agent(\historiesfiltering_{:t})
    \bigr),
$
we have
$
    \epsilon_\Joint(\historiesfiltering_{:t})
    =
    \chi\bigl(z(\historiesfiltering_{:t})\bigr).
$
Thus, in the finite case,
\begin{align}
    |\STATES^\Joint|
    \leq
    |\mathrm{Reach}(\STATES^\Environment \times \STATES^\Agent)|
    \leq
    |\STATES^\Environment|\,|\STATES^\Agent|.
\end{align}
This product bound is therefore a property of an explicitly specified factorised closed-loop presentation satisfying a joint-sufficiency condition.

\end{appendices}

\end{document}

%% file: math_commands.tex
\newtheorem{theorem}{Theorem}%[section]

\newtheorem{proposition}[theorem]{Proposition}

\theoremstyle{definition}
\newtheorem{definition}[theorem]{Definition}
\newtheorem{example}[theorem]{Example}
\newtheorem{remark}[theorem]{Remark}

\newcommand{\indep}{\perp\!\!\!\perp}

\newcommand{\N}{\mathbb{N}}
\newcommand{\Z}{\mathbb{Z}}

\newcommand{\prob}{p}

\newcommand{\Actions}{A}
\newcommand{\actions}{\MakeLowercase{\Actions}}
\newcommand{\ACTIONS}{\mathcal{\Actions}}

\newcommand{\supp}{\operatorname{supp}}

\DeclarePairedDelimiterX{\infdivx}[2]{(}{)}{%
  #1\;\delimsize\|\;#2%
}

\newcommand{\Environment}{E}

\newcommand{\Agent}{M}

\newcommand{\Joint}{J}

\newcommand{\States}{S}
\newcommand{\states}{\MakeLowercase{\States}}
\newcommand{\STATES}{\mathcal{\States}}

\newcommand{\Observations}{O}
\newcommand{\observations}{\MakeLowercase{\Observations}}
\newcommand{\OBSERVATIONS}{\mathcal{\Observations}}

\newcommand{\HistoriesFiltering}{H}
\newcommand{\historiesfiltering}{\MakeLowercase{\HistoriesFiltering}}
\newcommand{\HISTORIESFILTERING}{\mathcal{\HistoriesFiltering}}

%% file: 2cat.tikzstyles
\tikzstyle{function clear}=[fill=white, draw=dgreen, shape=circle, minimum size=6 pt, inner sep=2 pt]
\tikzstyle{function black}=[fill=black, draw=black, shape=circle, inner sep=1, minimum size=2mm]
\tikzstyle{function red}=[fill=red, draw=red, shape=circle, inner sep=1, minimum size=2mm]
\tikzstyle{function blue}=[fill=blue, draw=blue, shape=circle, inner sep=1, minimum size=2mm]
\tikzstyle{function magenta}=[fill=magenta, draw=magenta, shape=circle, inner sep=1, minimum size=2mm]
\tikzstyle{function teal}=[fill=teal, draw=teal, shape=circle, inner sep=1, minimum size=2mm]
\tikzstyle{function brown}=[fill=brown, draw=brown, shape=circle, inner sep=1, minimum size=2mm]
\tikzstyle{function orange}=[fill=orange, draw=orange, shape=circle, inner sep=1, minimum size=2mm]
\tikzstyle{new style 0}=[fill={rgb,255: red,76; green,169; blue,170}, draw={rgb,255: red,76; green,169; blue,170}, shape=circle, inner sep=1]
\tikzstyle{box-.5-.5}=[fill=none, draw=black, shape=rectangle, minimum width=5mm, minimum height=5mm, thick]
\tikzstyle{box-.5-1}=[fill=none, draw=black, shape=rectangle, minimum width=5mm, minimum height=10mm, thick]
\tikzstyle{box-1-.5}=[fill=none, draw=black, shape=rectangle, minimum width=1cm, minimum height=5mm, thick]
\tikzstyle{fbox-1-.5-blue}=[fill={blue!10}, draw=black, shape=rectangle, minimum width=1cm, minimum height=5mm, thick]
\tikzstyle{box-1-.5-blue}=[fill=none, draw=blue, shape=rectangle, minimum width=1cm, minimum height=5mm, thick]
\tikzstyle{fbox-1-.5-red}=[fill={red!10}, draw=black, shape=rectangle, minimum width=1cm, minimum height=5mm, thick]
\tikzstyle{box-1-.5-red}=[fill=none, draw=red, shape=rectangle, minimum width=1cm, minimum height=5mm, thick]
\tikzstyle{box-1-1}=[fill=none, draw=black, shape=rectangle, minimum width=1cm, minimum height=1cm, thick]
\tikzstyle{box-1-1-red}=[fill=none, draw=red, shape=rectangle, minimum width=1cm, minimum height=1cm, thick]
\tikzstyle{box-1-2}=[fill=none, draw=black, shape=rectangle, minimum width=1cm, minimum height=2cm, thick]
\tikzstyle{box-1.5-.5}=[fill=none, draw=black, shape=rectangle, minimum width=1.5cm, minimum height=5mm, thick]
\tikzstyle{box-1.5-1}=[fill=none, draw=black, shape=rectangle, minimum width=1.5cm, minimum height=1cm, thick]
\tikzstyle{box-2-.5}=[fill=none, draw=black, shape=rectangle, minimum width=2cm, minimum height=.5cm, thick]
\tikzstyle{box-2-1}=[fill=none, draw=black, shape=rectangle, minimum width=2cm, minimum height=1cm, thick]
\tikzstyle{box-3-.5}=[fill=none, draw=black, shape=rectangle, minimum width=3cm, minimum height=.5cm, thick]
\tikzstyle{box-3-2}=[fill=none, draw=black, shape=rectangle, minimum width=3cm, minimum height=2cm, thick]
\tikzstyle{box-3-4}=[fill=none, draw=black, shape=rectangle, minimum width=3cm, minimum height=4cm, rounded corners, thick]
\tikzstyle{box-4-.5}=[fill=none, draw=black, shape=rectangle, minimum width=4cm, minimum height=.5cm, thick]
\tikzstyle{box-4-3}=[fill=none, draw=black, shape=rectangle, minimum width=4cm, minimum height=3cm, thick]
\tikzstyle{dashed-box-4-3}=[fill=none, draw=black, shape=rectangle, minimum width=4cm, minimum height=3cm, thick, rounded corners, dashed]
\tikzstyle{box-4-5}=[fill=none, draw=black, shape=rectangle, minimum width=4cm, minimum height=5cm, thick]
\tikzstyle{box-5-2}=[fill=none, draw=black, shape=rectangle, minimum width=5cm, minimum height=2cm, thick]
\tikzstyle{box-5-3}=[fill=none, draw=black, shape=rectangle, minimum width=5cm, minimum height=3cm, thick]
\tikzstyle{box-5-4}=[fill=none, draw=black, shape=rectangle, minimum width=5cm, minimum height=4cm, thick]
\tikzstyle{box-6-3}=[fill=none, draw=black, shape=rectangle, minimum width=6cm, minimum height=3cm, thick]
\tikzstyle{box-6-4}=[fill=none, draw=black, shape=rectangle, minimum width=6cm, minimum height=4cm, thick]
\tikzstyle{box-6-5}=[fill=none, draw=black, shape=rectangle, minimum width=6cm, minimum height=5cm, thick]
\tikzstyle{box-7-5}=[fill=none, draw=black, shape=rectangle, minimum width=7cm, minimum height=5cm, thick]
\tikzstyle{box-8-4}=[fill=none, draw=black, shape=rectangle, minimum width=8cm, minimum height=4cm, thick]
\tikzstyle{box-10-6}=[fill=none, draw=black, shape=rectangle, minimum width=10cm, minimum height=6cm, thick]
\tikzstyle{box-13-6}=[fill=none, draw=black, shape=rectangle, minimum width=13cm, minimum height=6cm, thick]
\tikzstyle{ellipse-20-12}=[fill=none, draw=red, shape=ellipse, minimum width=20cm, minimum height=12cm, thick]
\tikzstyle{ellipse-14-8}=[fill=none, draw=yellow, shape=ellipse, minimum width=14cm, minimum height=8cm, thick]
\tikzstyle{ellipse-8-4.5}=[fill=none, draw=green, shape=ellipse, minimum width=8cm, minimum height=4.5cm, thick]

\tikzstyle{object red}=[-, draw={rgb,255: red,191; green,0; blue,64}, thick]
\tikzstyle{object blue}=[-, draw=blue, thick]
\tikzstyle{object orange}=[-, draw={rgb,255: red,255; green,128; blue,0}, thick]
\tikzstyle{big dashes}=[-, dash pattern=on 4mm off 2mm, thick, dashed]
\tikzstyle{arrow}=[->]
\tikzstyle{thick-line}=[-, thick]
\tikzstyle{thick-blue-line}=[-, thick, draw=blue]
\tikzstyle{thick-red-line}=[-, thick, draw=red]
\tikzstyle{thick-double-line}=[-, thick, double]
\tikzstyle{thick-arrow}=[->, thick]
\tikzstyle{thick-blue-arrow}=[->, thick, draw=blue]
\tikzstyle{thick-red-arrow}=[->, thick, draw=red]
\tikzstyle{tube-blue}=[-, fill=none, draw={rgb,255: red,103; green,161; blue,255}, line width=6]
\tikzstyle{tube-green}=[-, draw={rgb,255: red,110; green,207; blue,108}, line width=6, fill=none]